\documentclass[letterpaper]{article} % DO NOT CHANGE THIS
\usepackage{aaai25}  % DO NOT CHANGE THIS
\usepackage{times}  % DO NOT CHANGE THIS
\usepackage{helvet}  % DO NOT CHANGE THIS
\usepackage{courier}  % DO NOT CHANGE THIS
\usepackage[hyphens]{url}  % DO NOT CHANGE THIS
\usepackage{graphicx} % DO NOT CHANGE THIS
\usepackage{natbib}  % DO NOT CHANGE THIS AND DO NOT ADD ANY OPTIONS TO IT
\usepackage{caption} % DO NOT CHANGE THIS AND DO NOT ADD ANY OPTIONS TO IT
\usepackage{algorithm}
\usepackage{algorithmic}

\usepackage{newfloat}
\usepackage{listings}
\DeclareCaptionStyle{ruled}{labelfont=normalfont,labelsep=colon,strut=off} % DO NOT CHANGE THIS
\floatstyle{ruled}
\newfloat{listing}{tb}{lst}{}
\floatname{listing}{Listing}
\title{Coupling-based Convergence Diagnostic and Stepsize Scheme for Stochastic Gradient Descent}
\author {
    Xiang Li and Qiaomin Xie
}
\affiliations {
    Department of Industrial and Systems Engineering, University of Wisconsin–Madison\\
    xiang.li2@wisc.edu, qiaomin.xie@wisc.edu
}
\usepackage{bibentry}
\usepackage{amsthm}  % Example package for theorem-like environments
\newtheorem{assumption}{Assumption}
\newtheorem{theorem}{Theorem}
\newtheorem{claim}{Claim}
\newtheorem{proposition}{Proposition}
\newtheorem{lemma}{Lemma}

\usepackage{amssymb}
\usepackage{amsmath}
\usepackage{subfigure}

\newcommand{\N}{\mathbb{N}}
\newcommand{\R}{\mathbb{R}}
\newcommand{\E}{\mathbb{E}}
\newcommand{\CP}{\mathcal{P}}
\newcommand{\CB}{\mathcal{B}}

\global\long\def\norm#1{\left\Vert #1\right\Vert }%

\begin{document}

\maketitle

\begin{abstract}

The convergence behavior of Stochastic Gradient Descent (SGD) crucially depends on the stepsize configuration. When using a constant stepsize, the SGD iterates form a Markov chain, enjoying fast convergence during the initial transient phase. However, when reaching stationarity, the iterates oscillate around the optimum without making further progress. In this paper, we study the convergence diagnostics for SGD with constant stepsize, aiming to develop an effective dynamic stepsize scheme. We propose a novel \emph{coupling}-based convergence diagnostic procedure, which monitors the distance of two coupled SGD iterates for stationarity detection. Our diagnostic statistic is simple and is shown to track the transition from transience stationarity theoretically. 
We conduct extensive numerical experiments and compare our method against various existing approaches. Our proposed coupling-based stepsize scheme is observed to achieve superior performance across a diverse set of convex and non-convex problems. Moreover, our results demonstrate the robustness of our approach to a wide range of hyperparameters.   
\end{abstract}

\section{Introduction}\label{sec: Intro}

Stochastic Gradient Descent (SGD) aims to minimize an objective function by iteratively updating the model parameters based on noisy gradients of the function  \cite{robbins1951stochastic}. SGD is well suited for large-scale machine learning scenarios, where the data are abundant and the computation is limited, as it processes data points one at a time or in small batches \cite{bottou2010large, bottou2012stochastic}. 

Mathematically, the SGD recursion is given by
\begin{equation}\label{eq: 1}
\theta_{k+1} = \theta_k - \gamma_k \nabla f_{k}(\theta_k),
\end{equation}
Here, $\theta_k$ is the model parameters at iteration $k$,
$\nabla f_{k}(\theta_k)$ is the noisy gradient of the objective function $f$ at $\theta_k$, and $\gamma_k > 0$ is the stepsize (also known as the learning rate) for the $k$th iteration.

The stepsize configuration plays a critical role in the convergence of SGD. A traditional approach is to use diminishing stepsize, such as $\gamma_k=\gamma_0/k$ or $\gamma_0/\sqrt{k}$ \cite{robbins1951stochastic, bach2014adaptivity}. These decay stepsize schedules have been well studied, with non-asymptotic convergence guarantees \cite{moulines2011non, lacoste2012simpler, rakhlin2011making, shamir2013stochastic, JMLR:v15:hazan14a}. However, these schedules are often less robust to ill-conditioning and tend to suffer from a slow convergence rate. 

In practice, constant stepsize SGD is widely used, due to faster convergence and easy tuning. Prior work on constant stepsize SGD \cite{needell2014stochastic} provides a classical bias-variance decomposition of the mean-square error. Recent work provides a more precise characterization of its properties~\cite{dieuleveut2020bridging,yu2021nonconvex,lauand2023curse,merad2023sgd}, by leveraging the fact that the SGD iterates $(\theta_k)_{k\geq0}$ defined by \eqref{eq: 1} form a Markov chain with a constant stepsize $\gamma_k\equiv \gamma>0$. The Markov chain $(\theta_k)_{k\geq0}$ is shown to converge at an exponential rate that is proportional to the stepsize $\gamma.$ However, the chain does not converge to the optimal solution $\theta^\star$ but exhibits as a random walk within a vicinity of $\theta^\star$ of radius $\mathcal{O}(\sqrt{\gamma})$ as $k\rightarrow \infty$. 

The insights imply that fast convergence can be achieved by using a large constant stepsize; when the iterates reach stationarity,  reduce the stepsize and continue the iteration to achieve a better accuracy at the next saturation, upon which the procedure is repeated.  Indeed a common practice for SGD and deep learning algorithms is to use epoch-wise constant stepsize and periodically reduce it \cite{he2016deep, krizhevsky2012imagenet,wang2021convergence}, but the epoch lengths are mostly pre-determined and hand-picked. To make the best use of the procedure to achieve the best of both worlds, it is critical to develop a saturation diagnostic that effectively and promptly detects stationarity, while being computationally efficient. 

The problem of detecting saturation in constant stepsize SGD has been studied since the seminal work of \citet{pflug1983determination}, who proposed using the running average of the inner products of successive gradients to monitor convergence. Pflug's statistic has been used to develop several diagnostic algorithms for SGD \cite{chee2018convergence,yaida2018fluctuationdissipation}. However, recent work has shown that Pflug's method may fail even for quadratic functions, due to the large variance of the diagnostic statistic \cite{pesme2020distance}. In particular, the noisy signal fails to accurately indicate whether saturation has been reached, leading to frequent early stepsize reduction. Improving upon the diagnostic proposed by \citet{yaida2018fluctuationdissipation}, \citet{lang2019using} developed a more robust detection rule based on gradients of a mini-batch. Their approach, however, tends to suffer from the opposite issue, i.e., being overly conservative and leading delays in stationary detection. 

The more recent work in \citet{pesme2020distance} introduced an alternative method based on the distance between the initial iterate $\theta_0$ and the $k$th iterate $\theta_k$. 
The idea is that $\|\theta_k-\theta^\star\|^2$ is expected to reach saturation at approximately the same time as $\|\theta_k-\theta_0\|^2$. Their approach outperforms Pflug's methods empirically, but lacks theoretical justifications for general convex problems. A closely related work in \citet{sordello2020robust} considers splitting the SGD recursion into multiple threads, which are initialized from the same $\theta_0$ but use independent data points to calculate the gradients. Near-orthogonality of these gradients is used as an indicator of stationarity. Note that each thread only uses a fraction of the available data points.
 
In this paper, we propose a new, coupling-based convergence diagnostic procedure that is simple, flexible, and data-efficient. We apply this method to develop an effective dynamic stepsize scheme for SGD. Our main contributions are summarized as follows. First, building on the fact that the constant stepsize SGD iterates $(\theta_k)_{k\geq0}$ form a Markov chain, we design a stationarity diagnostic statistic via \emph{Markov chain coupling}. Specifically, our coupling-based method maintains two SGD iterates $(\theta_k^{(1)})_{k\geq0}$ and $(\theta_k^{(2)})_{k\geq0}$ using the same stepsize and data points at each iteration (``coupling"), but with different initialization. To assess whether the iterates have reached stationary phase, we perform a diagnosis based on the ratio of coupled iterates difference ${\|\theta^{(1)}_{k} - \theta^{(2)}_{k}\|^2}/{\|\theta^{(1)}_{0} - \theta^{(2)}_{0}\|^2}.$ Once stationarity is detected, the stepsize is reduced, and the procedure is repeated. 
Our stationarity detection is simple and easy to implement. Moreover, we prove that our diagnostic statistic is valid for convergence detection in general convex problems. Furthermore, we conduct extensive experiments to evaluate our coupling-based dynamic stepsize scheme against various existing methods. Our approach consistently achieves superior performance across a range of convex and non-convex problems, such as logistic regression and ResNet-18, as well as in stochastic approximation with Markovian data. Lastly, our results demonstrate the robustness of our approach to a wide range of hyper-parameters.

\section{Preliminaries}\label{sec:preliminary}

In this paper, we focus on the classical SGD algorithm for minimizing an objective function $f: \mathbb{R}^d \rightarrow \mathbb{R}.$ With an initialization $\theta_0\in \R^d$, the SGD recursion~\eqref{eq: 1} can be written equivalently as 
\begin{equation}\label{eq: SGD}
\theta_{k+1} = \theta_k - \gamma_k \big[\nabla f(\theta_k) + \varepsilon_{k}(\theta_k)\big],
\quad k\geq 0,
\end{equation}
where $\varepsilon_k:\R^d \rightarrow \R^d$ is the random field corresponding to the 
stochasticity in $\nabla f_k (\cdot)$ as an estimate of the true gradient $\nabla f(\cdot)$, such that $\nabla f_k (\cdot) = \nabla f (\cdot)+ \varepsilon_k (\cdot).$

For our analysis, we consider the following assumptions on the loss function $f$ and the noise functions $(\varepsilon_k)_{k\ge0}$.

    \begin{assumption}[Smoothness]\label{assumption: L-smooth}
    The function $f$ is $L$-smooth with $L \geq 0$, i.e.,  for all $\theta,\theta' \in \mathbb{R}^d$,
    \begin{equation*}
    \|\nabla f(\theta)-\nabla f(\theta')\| \leq L\|\theta-\theta'\|.
    \end{equation*}
    \end{assumption}

    \begin{assumption}[Strong convexity]\label{assumption: convex} The function $f$ is strongly convex with parameter $\mu>0,$ i.e., $\forall\theta,\theta' \in \mathbb{R}^d$,
    \begin{equation*}
    \langle \nabla f(\theta)-\nabla f(\theta'),\theta -\theta' \rangle\geq \mu\|\theta - \theta'\|^2.
    \end{equation*}
    \end{assumption}

Assumption \ref{assumption: convex} implies that the objective function $f$ admits a unique global optimum $\theta^\star.$

    \begin{assumption}[Zero-mean noise]\label{assumption: unbiased}
    There exists a filtration $(\mathcal{F}_k)_{k\geq 0}$ on some probability space such that for all $k\in \mathbb{N}$ and $\theta \in \R^d,$ $\varepsilon_{k}(\theta)$ is $\mathcal{F}_k$-measurable and $\E[\varepsilon_k(\theta)\mid \mathcal{F}_{k-1}]=0.$ In addition, $(\varepsilon_k)_{k\geq 0}$ are independent and identically distributed random fields. 
    \end{assumption}

Assumption \ref{assumption: unbiased} implies that $\nabla f_{k}(\theta):= \nabla f(\theta)+\varepsilon_k(\theta)$ is an unbiased estimator of the true gradient $\nabla f(\theta)$. 

    \begin{assumption}[Co-coercivity]\label{assumption: bounded_variance}
    The gradient is co-coercive in expectation, i.e., $ \langle \nabla f(\theta) - \nabla f(\theta'), \theta - \theta' \rangle \ge (1/L) \mathbb{E}[ \| \nabla f_k(\theta) - \nabla f_k(\theta') \|^2], \forall k>0,\theta,\theta' \in \R$, where $L$ is given in Assumption~\ref{assumption: L-smooth}. Moreover, there exists a constant $\sigma \geq 0$ such that for all $k > 0$, $\mathbb{E}[\|\varepsilon_k(\theta^\star)\|^2] \leq \sigma^2$. 
    \end{assumption}

Assumption~\ref{assumption: bounded_variance} is standard and appeared in \citet{dieuleveut2020bridging, merad2023sgd}.

\paragraph{Transience-stationarity under constant stepsize.}
Constant stepsize SGD has recently gained increasing attention and its properties have been well understood for well-conditioned problems. In particular, the classical result \cite{needell2014stochastic} decomposes the bound on $\mathbb{E}\left[\|\theta_k - \theta^\star\|^2\right]$ into bias and variance terms: the bias term is proportional to the initial condition $\|\theta_0-\theta^\star\|$ and decays at an exponential rate; the variance term is determined by the gradient noise. In general, the iterates are not converging to the global optimum $\theta^\star$.

Recent work has provided a more precise characterization of the bias-variance trade-off through the lens of Markov chain theory \cite{dieuleveut2020bridging,yu2021nonconvex}. 
In particular, with a constant stepsize $\gamma_k \equiv \gamma>0,$ the SGD iterates $(\theta_k)_{k\geq 0}$ given by \eqref{eq: SGD} form a Markov chain. The Markov chain is shown to converge to a unique stationary distribution $\pi_{\gamma}$ under appropriate conditions. Let $\CP_2(\R^d)$ denote the set of probability measure on $(\R^d,\CB(\R^d))$ with a finite second moment, where $\CB(\R^d)$ is the Borel $\sigma$-field of $\R^d.$

\begin{proposition}
 [Proposition 2 in \citealt{dieuleveut2020bridging}]
 \label{prop:MC_convergence}
Suppose that Assumptions \ref{assumption: L-smooth}--\ref{assumption: bounded_variance} hold. With constant stepsize $\gamma \in (0,2/L),$ the Markov chain $(\theta_k)_{k\geq 0}$ given by the recursion \eqref{eq: SGD} satisfies:
\begin{equation}\label{eq: W2}
W_2^2\big(P^k_\gamma(\theta_0,\cdot),\pi_\gamma \big) \leq \rho^k \E_{\theta\sim\pi_{\gamma}}\big[\norm{\theta_0-\theta}^2\big],
\end{equation}
where $P^k_\gamma$ is the $k$-step Markov kernel for the chain $(\theta_k)_{k\geq 0}$, $W_2(\nu,\nu')$ is the Wasserstein distance of order two between measures $\nu,\nu'\in \CP_2(\R^d),$ and $\rho:=1-2\gamma \mu (1-\gamma L/2).$
\end{proposition}

Proposition \ref{prop:MC_convergence} shows that during the transient phase, i.e., before reaching stationarity, the Markov chain converges exponentially fast at a rate proportional to the stepsize $\gamma.$ Consequently, using a larger stepsize achieves faster transient convergence. However, when reaching stationary phase, the iterates exhibit a random walk around $\theta^\star$, incurring a \emph{saturated} expected error $\E_{\pi_\gamma}[\theta]-\theta^\star=A\gamma+O(\gamma^2)$ \cite[Theorem 4]{dieuleveut2020bridging}. Therefore, a larger stepsize suffers from a larger approximation error. Similar transience-stationarity properties have also been characterized for general stochastic approximation algorithms with Markovian data, which covers SGD as a special case \cite{lauand2023curse,huo2024collusion}.

This above insight on the transient phase convergence rate and the stationary phase error naturally indicates an \emph{adaptive constant stepsize scheduling} procedure to achieve the best of both worlds: Start with a large constant stepsize for fast transient convergence; when the iterates have approximately reached stationarity, decrease the stepsize and continue the SGD process to achieve smaller approximation error. 

The key challenge here is how to detect the stationarity of the iterates effectively, which has remained an open problem in the literature \cite{pasupathy2019open_problem}. 
Several statistical tests have been proposed for stationarity diagnostic \cite{chee2018convergence,yaida2018fluctuationdissipation,pesme2020distance}. However, these methods either lead to poor empirical results or lack of theoretical analysis. There is a critical need for a more efficient and principled tool to assess the convergence of SGD iterates with constant stepsize. In the next section, we present our approach that uses Markov chain coupling algorithmically.

\section{A Coupling-based Statistic for Stationarity Diagnostic}\label{sec:algorithm}

In this section, we introduce a novel coupling-based statistic for convergence diagnostic and present the theoretical foundation for justify its effectiveness. 

Our approach builds on the fact that the SGD iterates $(\theta_k)_{k\geq 0}$ forms a Markov chain under a constant stepsize. We propose a diagnostic that detects the convergence of the Markov chain $(\theta_k)_{k\geq 0}$ by \emph{coupling}. 
Stochastic coupling techniques have served as a powerful \emph{analytical} tool for characterizing stochastic system performance. In particular, it has been widely employed to establish Markov chain convergence \cite{dieuleveut2020bridging,huo2023bias,foss-coupling,lauand2023curse}, which serves as an inspiration for our approach. 

In particular, consider two  SGD iterates $(\theta_k^{(1)})_{k\geq 0}$ and $(\theta_k^{(2)})_{k\geq 0}$ that use the same stepsize $\gamma$ and the same mini-batch data at every iteration, i.e., \emph{sharing the same noise}, but with different initialization $\theta_0^{(1)}$ and $\theta_0^{(2)}$. This coupling of the two Markov chains $(\theta_k^{(1)})_{k\geq 0}$ and $(\theta_k^{(2)})_{k\geq 0}$ has played an important role in showing the convergence of the iterates to their unique stationary distribution $\pi_{\gamma}$; e.g., see the proof of Proposition 2 by \citet{dieuleveut2020bridging}. 

Here we propose to employ the coupling technique algorithmically to construct a statistical test for stationarity diagnostic.  
Let $D_{k}:=\theta_{k}^{(1)}-\theta_{k}^{(2)}$ denote their difference. Intuitively, the distributional convergence of the iterates $(\theta_{k}^{(1)})_{k\geq0}$ and $(\theta_{k}^{(2)})_{k\geq0}$ implies that $\norm{D_{k}}\approx0$ if $\theta_{k}^{(1)}$ and $\theta_{k}^{(2)}$ reach the stationary phase, thanks to the coupling procedure. This observation suggests a natural criterion for the stationarity of the primary sequence $\theta_{k}^{(1)}$ via tracking the ratio of the differences, ${\norm{D_{k}}}/{\norm{D_{0}}}$, with the auxiliary sequence $\theta_{k}^{(2)}$. This naturally leads to the following convergence diagnostic: if $\norm{D_{k}}/\norm{D_{0}}$ is smaller than some threshold, then the stepsize is decreased by a factor of $r\in(0,1)$, upon which we reinitialize one sequence $\theta_{k}^{(2)}$ and repeat this procedure. 

Our intuition is clearly illustrated in Figure \ref{fig:Dk}. Importantly, we observe that when ratio of the distance $\norm{D_{k}}/\norm{D_0}$ decreases to a small value, the SGD iterates $(\theta_k^{(1)})_{k\geq0}$ approximately enter the stationary phase. 
A natural question arises: Can we theoretically justify that the behavior of the difference $D_k$ accurately reflects the transition of iterates $(\theta^{(1)}_k)_{k\geq 0}$ between transience and stationarity? 

\begin{figure}[htbp]
\centering
\includegraphics[width=0.45\textwidth]{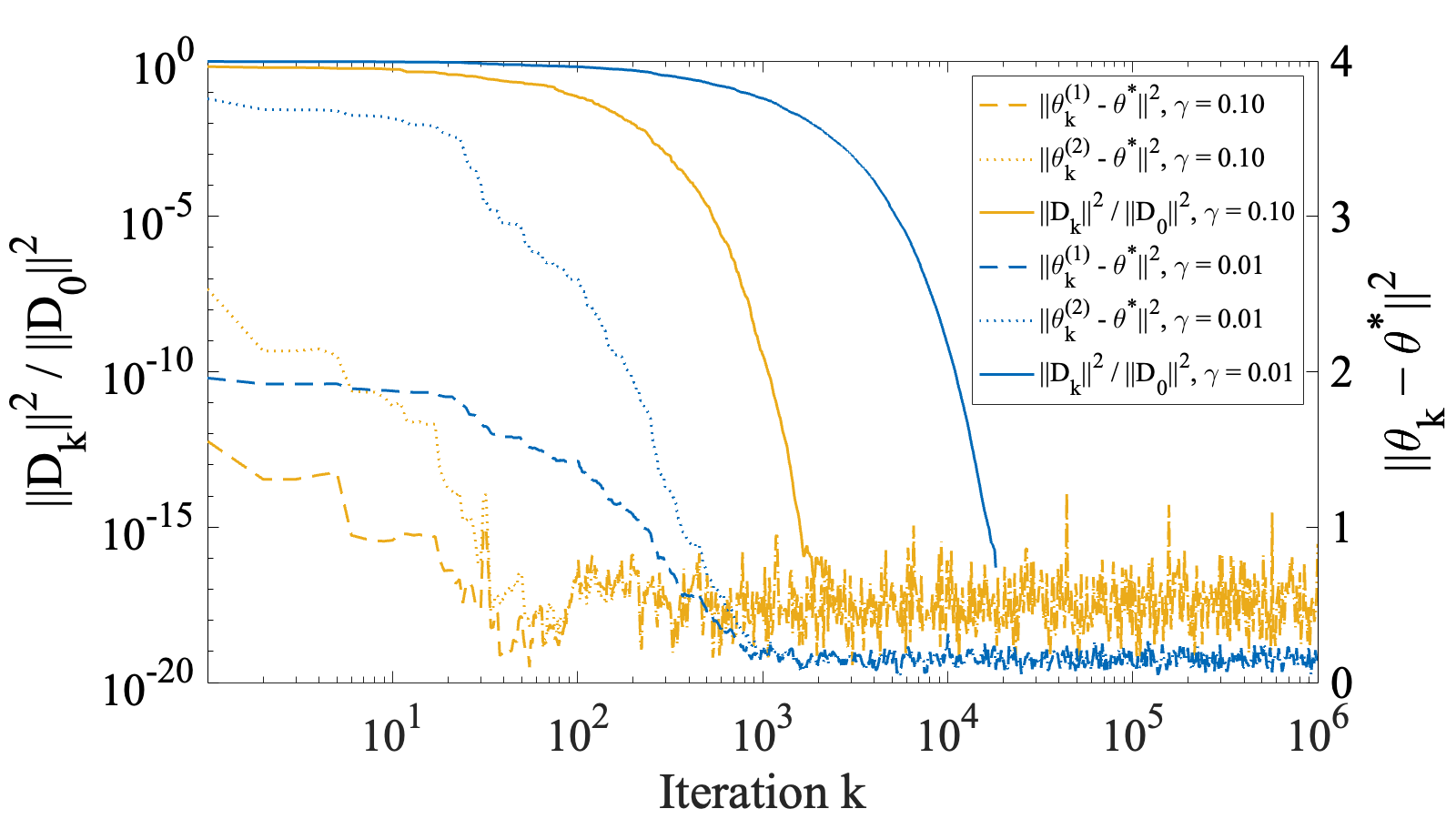}
\caption{Evolution of $\|\theta^{(1)}_k - \theta^\star\|^2$ and $\|\theta^{(2)}_k - \theta^\star\|^2$  and $\|\theta^{(1)}_k - \theta^{(2)}_k\|^2$ under least squares regression with two different constant stepsizes.}
\label{fig:Dk}
\end{figure}

\subsubsection{Quadratic Setting}

To answer the above question for our coupling approach, we first consider the case of a quadratic objective function, as stated in the following assumption. 

    \begin{assumption}[Quadratic semi-stochastic setting]\label{assumption: Quadratic}
    There exists a symmetric positive definite matrix $H$, a vector $a$ and a constant $c$ such that $f(\theta) = \frac{1}{2}\theta^\top H \theta + a^\top \theta + c$. The gradient noise $\varepsilon_i(\theta) = \xi_i$ is independent of $\theta$, where  $(\xi_i)_{i \geq 0}$ are i.i.d and satisfy $\mathbb{E}[\xi_i] = 0$ and $\mathbb{E}[\xi_i^\top \xi_i] = C$.
    \end{assumption}

Under this assumption, we can compute $\mathbb{E}[\|\theta^{(1)}_{k} - \theta^{(2)}_{k}\|^2]$ in a closed-form, as stated in the following proposition. The proof is provided in Appendix \ref{App: Proofs}

\begin{proposition}\label{proposition: Dk_quadratic}
 Suppose Assumptions \ref{assumption: L-smooth}, \ref{assumption: convex}, and \ref{assumption: Quadratic} hold. Let $\gamma \in \left(0, {1}/{L}\right)$. We have that for all $k \geq 0$:
\begin{align*}
    &\mathbb{E}[\|D_k\|^2] = \E[D_0^\top(I-\gamma H)^{2k}D_0]\leq (1-\gamma \mu)^{2k}\E[\|D_0\|^2].
\end{align*}
\end{proposition}

Proposition \ref{proposition: Dk_quadratic} states that the expected squared distance decays exponentially and eventually converges to zero,

\begin{proposition}[Quadratic]\label{proposition: proximity-quadratic}
Under Assumptions \ref{assumption: L-smooth}, \ref{assumption: convex}, and \ref{assumption: Quadratic}, we have the following result with constant stepsize $\gamma \in (0,1/L):$
\begin{align}
&W_2^2\big(P^k_\gamma(\theta_0^{(1)},\cdot),\pi_\gamma \big)\nonumber \\
 \leq & \frac{\E[\|D_k\|^2]}{\E[\big(D_0^{\top}q_{\max}\big)^2]}\cdot\E_{\theta\sim\pi_{\gamma}}\big[\|\theta_0^{(1)}-\theta\|^2\big], 
\end{align}
where $q_{\max}$ is the eigenvector associated with the largest eigenvalue of the matrix $I-\gamma H.$
\end{proposition}

\begin{proof}
By centering, we may assume without loss of generality that $a=0$ and $c=0$, so $f(\theta) = \frac{1}{2}\theta^\top H \theta$. For this setting, we have the following claim on the distributional convergence of the Markov chain $(\theta_k)_{k\geq 0}.$ The proof of the claim follows similar lines as that of Proposition 2 in \citet{dieuleveut2020bridging}. We provide the proof in Appendix \ref{App: Proofs} for completeness.

\begin{claim}\label{claim:W2_quadratic}
Under the setting of Proposition \ref{proposition: proximity-quadratic}, the Markov chain $(\theta_k)_{k\geq 0}$ given by the recursion \eqref{eq: SGD} satisfies:
\begin{equation}\label{eq: W2_quadratic}
W_2^2\big(P^k_\gamma(\theta_0,\cdot),\pi_\gamma \big) \leq (1-\gamma \lambda_{\min})^{2k} \E_{\theta\sim\pi_{\gamma}}\big[\norm{\theta_0-\theta}^2\big],
\end{equation}
where $\lambda_{\min}$ being the smallest eigenvalue of $H.$
\end{claim}

Note that the largest eigenvalues of the matrix $(I-\gamma H)$ is $1-\gamma \lambda_{\min}.$ Let $q_{\max}$ be the corresponding eigenvector. 
By Proposition \ref{proposition: Dk_quadratic}, we have:
\begin{align*}
\E[\|D_k\|^2] &= \E[D_0^\top(I-\gamma H)^{2k}D_0]\nonumber\\
&\geq (1-\gamma\lambda_{\min})^{2k}\E[\big(D_0^{\top}q_{\max}\big)^2].
\end{align*}
If the initial difference vector $D_0$ satisfies $\E[\big(D_0^{\top}q_{\max}\big)^2]\neq0$, 
combining the above inequality with the inequality \eqref{eq: W2_quadratic} from Claim \ref{claim:W2_quadratic} yields the desired result.  
\end{proof}

From Proposition \ref{proposition: proximity-quadratic}, we note that a small distance between the two iterates $\theta^{(1)}_{k},\theta^{(2)}_{k}$ implies that the iterates $\theta^{(1)}_{k}$ approximately converges to its stationary distribution $\pi_{\gamma}.$ Since the eigenvector $q_{\max}$ is unknown, we propose to track the ratio of the iterates difference $\norm{D_{k}}/\norm{D_{0}}$
as an approximation, which is easy to compute.  

\subsubsection{General Convex Setting}
We next generalize our analysis to non-quadratic setting.

\begin{theorem}[General Convex]
\label{thm: proximity-general}
Suppose Assumptions \ref{assumption: L-smooth}--\ref{assumption: bounded_variance} hold and $\gamma \in (0,\gamma_0],$ where \( \gamma_0 = \min \big\{ \frac{1}{4L}, \frac{2L}{\mu} \big\} \). Then $\forall \tau \in [\frac{4L}{\mu}, \infty)$, we have the following result:
\begin{align*}
&W_2^2\big(P^k_\gamma(\theta_0^{(1)},\cdot),\pi_\gamma \big) \\
\leq & \left( \frac{\E[\|D_k\|^2]}{\E[\|D_0\|^2]} \right)^{1/\tau} \cdot \E_{\theta\sim\pi_{\gamma}}\big[\|{\theta_0^{(1)}-\theta}\|^2\big].
\end{align*}
\end{theorem}

\begin{proof}
    Consider the SGD update rule in eq. (\ref{eq: 1}) and fix an arbitrary integer $k\ge 0.$
    The expected squared norm of the parameter difference after one iteration is given by:
    \begin{align}\label{eq:5}
        &\mathbb{E}\|\theta_{k+1}^{(1)} - \theta_{k+1}^{(2)}\|^2 \\
        &= \mathbb{E} \| (\theta_{k}^{(1)} - \gamma \nabla f_{k}(\theta_{k}^{(1)})) - (\theta_{k}^{(2)} - \gamma \nabla f_{k}(\theta_{k}^{(2)}))\|^2 \nonumber \\
        &= \mathbb{E}\|\theta_{k}^{(1)} - \theta_{k}^{(2)}\|^2 + \gamma^2 \mathbb{E}\|\nabla f_{k}(\theta_{k}^{(1)}) - \nabla f_{k}(\theta_{k}^{(2)})\|^2\nonumber \\
        &\quad - 2\gamma \mathbb{E} \langle \nabla f_{k}(\theta_{k}^{(1)}) - \nabla f_{k}(\theta_{k}^{(2)}), \theta_{k}^{(1)} - \theta_{k}^{(2)} \rangle. \nonumber 
    \end{align}
By the smoothness assumption of \(f\), we have:
    \begin{align*}
&\langle \nabla f(\theta_{k}^{(1)}) - \nabla f(\theta_{k}^{(2)}), \theta_{k}^{(1)} - \theta_{k}^{(2)} \rangle \leq L \|\theta_{k}^{(1)} - \theta_{k}^{(2)}\|^2.
    \end{align*}
Therefore, we can lower bound the LHS of Equation (\ref{eq:5}) as
    \begin{align*}
    \mathbb{E}\|\theta_{k+1}^{(1)} - \theta_{k+1}^{(2)}\|^2 &\geq \mathbb{E}\|\theta_{k}^{(1)} - \theta_{k}^{(2)}\|^2 - 2\gamma L\mathbb{E}\|\theta_{k}^{(1)} - \theta_{k}^{(2)}\|^2 \\
    &\quad + \gamma^2 \mathbb{E}\|\nabla f_1(\theta_{k}^{(1)}) - \nabla f(\theta_{k}^{(2)})\|^2.
    \end{align*}
    Moreover, by strong convexity of $f$ together with  the Cauchy-Schwarz inequality, we obtain
    \[
    \|\nabla f(\theta_{k}^{(1)}) - \nabla f(\theta_{k}^{(2)})\| \geq \mu \|\theta_{k}^{(1)} - \theta_{k}^{(2)}\|.
    \]
    It follows that
    \begin{equation}
       \mathbb{E}\|\theta_{k+1}^{(1)} - \theta_{k+1}^{(2)}\|^2 \geq (1 - 2\gamma L + \gamma^2 \mu^2) \mathbb{E}\|\theta_{k}^{(1)} - \theta_{k}^{(2)}\|^2 \nonumber.
    \end{equation}
    Therefore, by induction on $k$, we obtain the lower bound:
    \begin{equation}
        \mathbb{E}\|\theta_k^{(1)} - \theta_k^{(2)}\|^2 \geq \varrho^k \mathbb{E}\|\theta_0^{(1)} - \theta_0^{(2)}\|^2,
    \end{equation}
    where $\varrho := 1 - 2\gamma L + \gamma^2 \mu^2.$ 
    Applying Lemma \ref{lemma: 1} from Appendix \ref{sec:tech_lemma} yields:
    \begin{equation*}
        {\mathbb{E}\|D_k\|^2}/{\E\|D_0\|^2} \geq \varrho^{k} \geq (1-\gamma\mu)^{k\cdot k_0} \geq (1-\gamma\mu)^{k\cdot \tau},
    \end{equation*}
    where the last inequality follows from the fact that $\tau \geq k_0:=4L/\mu.$  Combining with Proposition~\ref{prop:MC_convergence}, we obtain
\begin{align*}
&\quad W_2^2\big(P^k_\gamma(\theta_0,\cdot),\pi_\gamma \big) \leq \rho^k \E_{\theta\sim\pi_{\gamma}}\big[\norm{\theta_0-\theta}^2\big] \\
& \leq (1-\gamma \mu)^{k} \E_{\theta\sim\pi_{\gamma}}\big[\norm{\theta_0-\theta}^2\big] \\
&\leq \left( {\E\|D_k\|^2}/{\E\|D_0\|^2} \right)^{1/\tau} \cdot \E_{\theta\sim\pi_{\gamma}}\big[\norm{\theta_0-\theta}^2\big].
\end{align*}
\end{proof}

Built on Proposition \ref{proposition: proximity-quadratic} and Theorem~\ref{thm: proximity-general}, we conclude that the ratio of distances, ${\|D_k\|^2}/{\|D_0\|^2}$ naturally serves as an efficient statistic that detects the transition of the Markov chain from transience to stationarity. In particular, when the statistic falls below a predefined threshold, it indicates that $\theta_k^{(1)}$ approximately saturates. Then the algorithm triggers a stepsize decay and repeats. 

We formally describe our method in Algorithm~\ref{alg:static}, named \textit{Coupling-based SGD}.
Our diagnostic algorithm aims to detect the saturation of constant stepsize SGD $(\theta_k^{(1)})_{k\geq0}$ by tracking its distance to an auxiliary SGD sequence $(\theta_k^{(2)})_{k\geq0}$. Given two initial points, two SGD iterates progress simultaneously using the same stepsize $\gamma$ and mini-batch of data. At each iteration $k$, we calculate the ratio of distances, ${\|\theta^{(1)}_{k} - \theta^{(2)}_{k}\|^2}/{\|\theta^{(1)}_{0} - \theta^{(2)}_{0}\|^2}$. If the ratio falls below a certain threshold $\beta,$ the stepsize is decreased by a factor $r$, and the auxiliary iterate $\theta_k^{(2)}$ is \emph{re-initialized.} Setting an arbitrary re-initialization might take the auxiliary iterate far away from the vicinity of the optimum solution $\theta^{\star}.$ Consequently, it would take the auxiliary iterates $\theta_k^{(2)}$ long time to re-entering the vicinity and sync with the primary iterates $\theta_k^{(1)}.$ To address this issue, we propose to 
set the new initial point using the $b$-step backward iterate $\theta_{k-b}^{(2)}$, which stays close to the stationary neighborhood with an appropriate $b$. Our experiment in Section \ref{sec: Robustness Results} demonstrates the robustness of our approach to the parameter $b$. 

\begin{algorithm}[tb]
\caption{\, Coupling-based SGD}
\label{alg:static}
\textbf{Input}: $\theta^{(1)}_0$, $\theta^{(2)}_0$\\
\textbf{Parameters}: initial stepsize $\gamma$, stepsize decay factor $r$, backward  steps~$b$, threshold $\beta$
\begin{algorithmic}[1]
\WHILE{$k \in \{1, 2, 3, \dots, n\}$}
\STATE $\theta^{(1)}_k \gets \theta^{(1)}_{k-1} - \gamma \nabla f_{k-1}(\theta^{(1)}_{k-1})$
\STATE $\theta^{(2)}_k \gets \theta^{(2)}_{k-1} - \gamma \nabla f_{k-1}(\theta^{(2)}_{k-1})$
    \STATE $S \gets {\|\theta^{(1)}_{k} - \theta^{(2)}_{k}\|^2}/{\|\theta^{(1)}_{0} - \theta^{(2)}_{0}\|^2}$
    \COMMENT{distance between two SGD sequences $\thickapprox\rho^k$}
        \IF{$S < \beta$}
        \STATE $\gamma \gets r\cdot\gamma$
        \COMMENT{reduce the stepsize}
        \STATE $\theta^{(2)}_{k} \gets \theta^{(2)}_{k-b}$
        \COMMENT{update the new initial point}
        \ENDIF
\ENDWHILE
\end{algorithmic}
\end{algorithm}

\paragraph{Adaptive Coupling-based SGD}
As illustrated in Figure \ref{fig:Dk}, as the stepsize $\gamma$ gets smaller, the iterates in stationarity incur smaller approximation error and variance. Indeed, the stationary distribution $\pi_\gamma$ is shown to satisfy $\E_{\pi_\gamma}[\theta]-\theta^\star=A\gamma+O(\gamma^2)$ and $\E_{\pi_\gamma}[\|\theta-\theta^\star\|^2]=A'\gamma+O(\gamma^2)$ \cite[Theorem 4]{dieuleveut2020bridging}. With a smaller $\gamma,$ the stationary distribution $\pi_{\gamma}$ lives in a smaller vicinity of $\theta^{\star}.$ Consequently, to accurately detect distributional convergence of the iterates, it makes sense to employ a more stringent criterion with smaller stepsize $\gamma$. To this end, we propose using an adaptive threshold $\beta$ for the diagnostic statistic, as presented in Algorithm \ref{alg:adaptive}. In particular, as $\gamma$ decreases, we also reduce the threshold $\beta$ by a factor $\eta\in (0,1).$ 

\begin{algorithm}[tb]
\caption{\, Adaptive Coupling-based SGD}
\label{alg:adaptive}
\textbf{Input}: $\theta^{(1)}_0$, $\theta^{(2)}_0$\\
\textbf{Parameter}: initial stepsize $\gamma$, stepsize decay factor $r$,  backward steps $b$, initial threshold $\beta$, threshold decay factor $\eta.$
\begin{algorithmic}[1]
\WHILE{$k \in \{1, 2, 3, \dots, n\}$}
\STATE $\theta^{(1)}_k \gets \theta^{(1)}_{k-1} - \gamma \nabla f_{k-1}(\theta^{(1)}_{k-1})$
\STATE $\theta^{(2)}_k \gets \theta^{(2)}_{k-1} - \gamma \nabla f_{k-1}(\theta^{(2)}_{k-1})$
    \STATE $S \gets {\|\theta^{(1)}_{k} - \theta^{(2)}_{k}\|^2}/{\|\theta^{(1)}_{0} - \theta^{(2)}_{0}\|^2}$
    \COMMENT{distance between two SGD sequences $\thickapprox\rho^k$}
        \IF{$S < \beta$}
        \STATE $\gamma \gets r \cdot \gamma$
        \COMMENT{reduce the stepsize}
        \STATE $\beta \gets \eta\cdot\beta$
        \COMMENT{reduce the threshold}
        \STATE $\theta^{(2)}_{k} \gets \theta^{(2)}_{k-b}$
        \COMMENT{update the new initial point}
        \ENDIF
\ENDWHILE
\end{algorithmic}
\end{algorithm}

\section{Empirical Study}\label{sec: Experiments}
In this section, we present an empirical study of our proposed diagnostic method on various tasks. 

\subsection{Setup} 

\noindent\textbf{Algorithms.} Our main baselines include the Pflug's diagnostic-based approach $\text{ISGD}^{1/2}$ \cite{chee2018convergence}, and the distance-based algorithm \cite{pesme2020distance}. For these baseline methods, we carefully tune their parameters to achieve their best performance. 

{To demonstrate the effectiveness and robustness of our proposed diagnostic method, we conduct extensive experiments on a diverse set of problems, including (1) Logistic regression, (2) Least squares regression, (3) The 18-layer ResNet model \cite{he2016deep}, (4) Linear stochastic approximation with Markovian data, (5) SVM, (6) Uniformly convex functions, (7) Lasso.} Due to space constraint, here we focus on the experiment setup and results of the first three tasks and defer details of other settings to Appendix \ref{App: Additional_Experiments}\footnote{The code is available at \url{https://github.com/XianggLi/coupling-based-sgd}}. 

\begin{figure*}[htbp]
    \centering
    \subfigure{
        \includegraphics[width=0.41\textwidth]{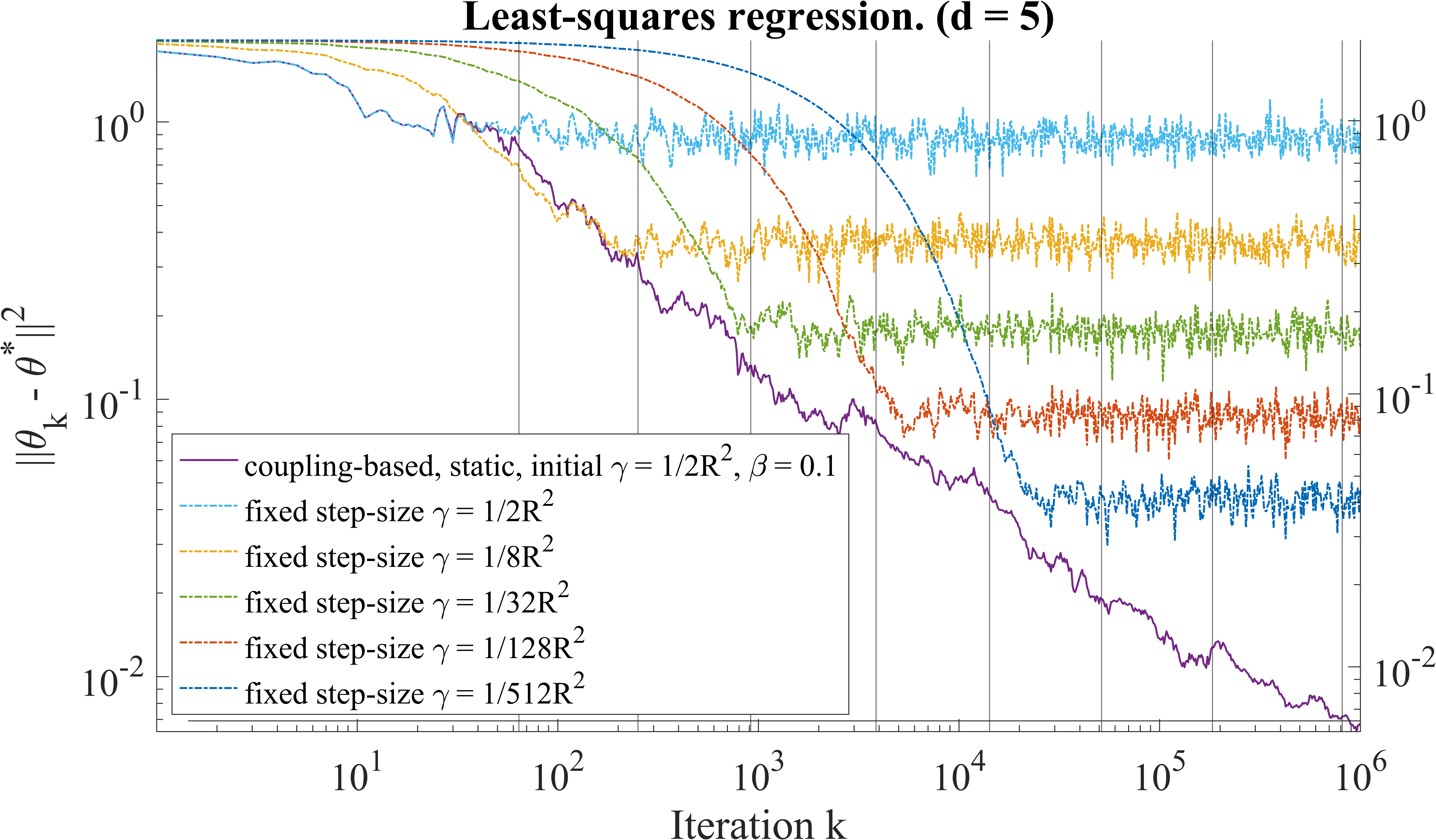}
    }
    \subfigure{
        \includegraphics[width=0.41\textwidth]{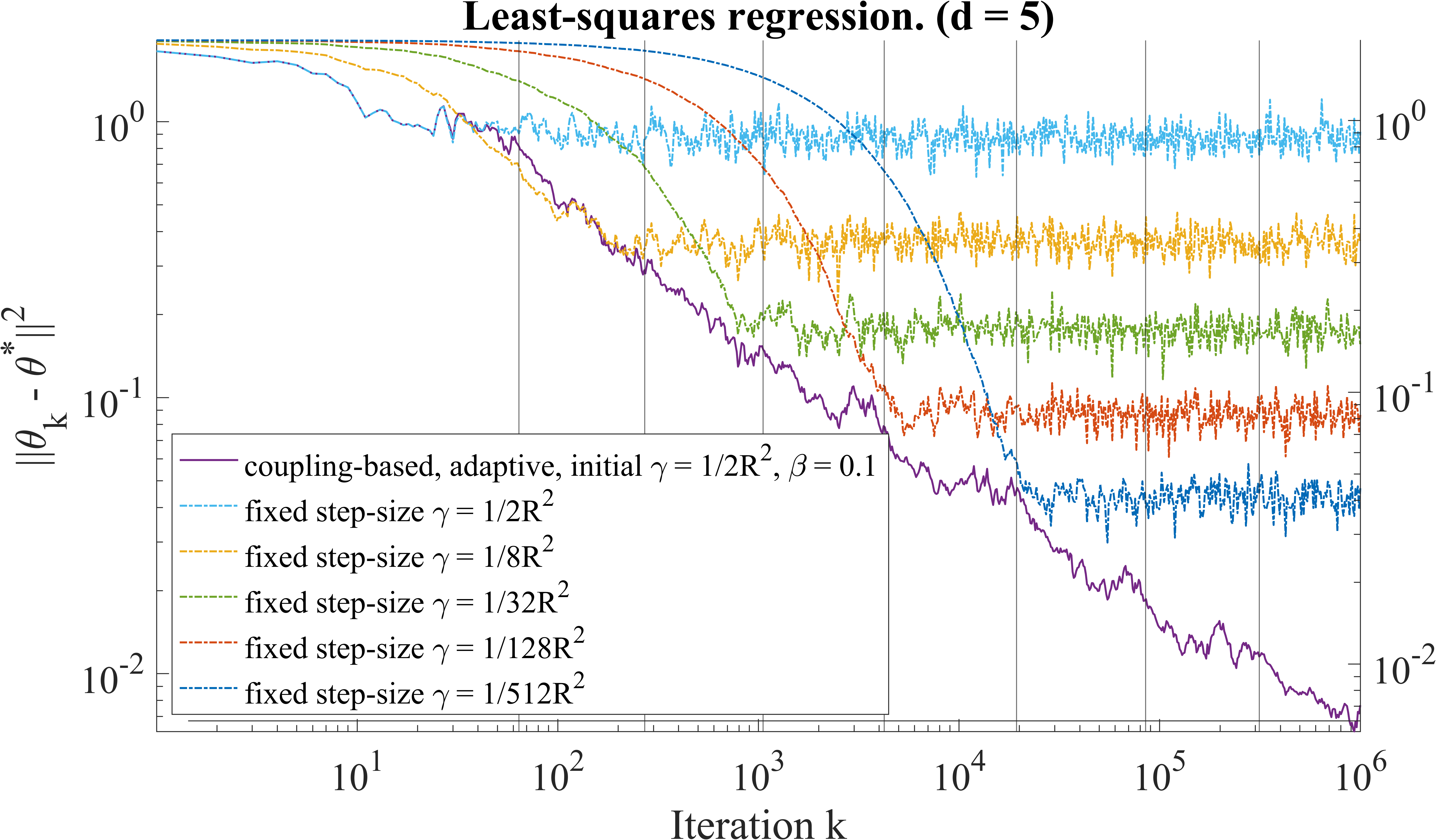}
    }
    \caption{Effectiveness of our coupling-based statistic  for stationarity diagnostic. Left: Algorithm \ref{alg:static} with static threshold; Right: Algorithm \ref{alg:adaptive} with adaptive threshold. The vertical lines correspond to restarts of our coupling-based algorithms.}
    \label{fig:constant_step}
\end{figure*}

\begin{figure*}[htb]
    \centering
    \subfigure{
        \includegraphics[width=0.41\textwidth]{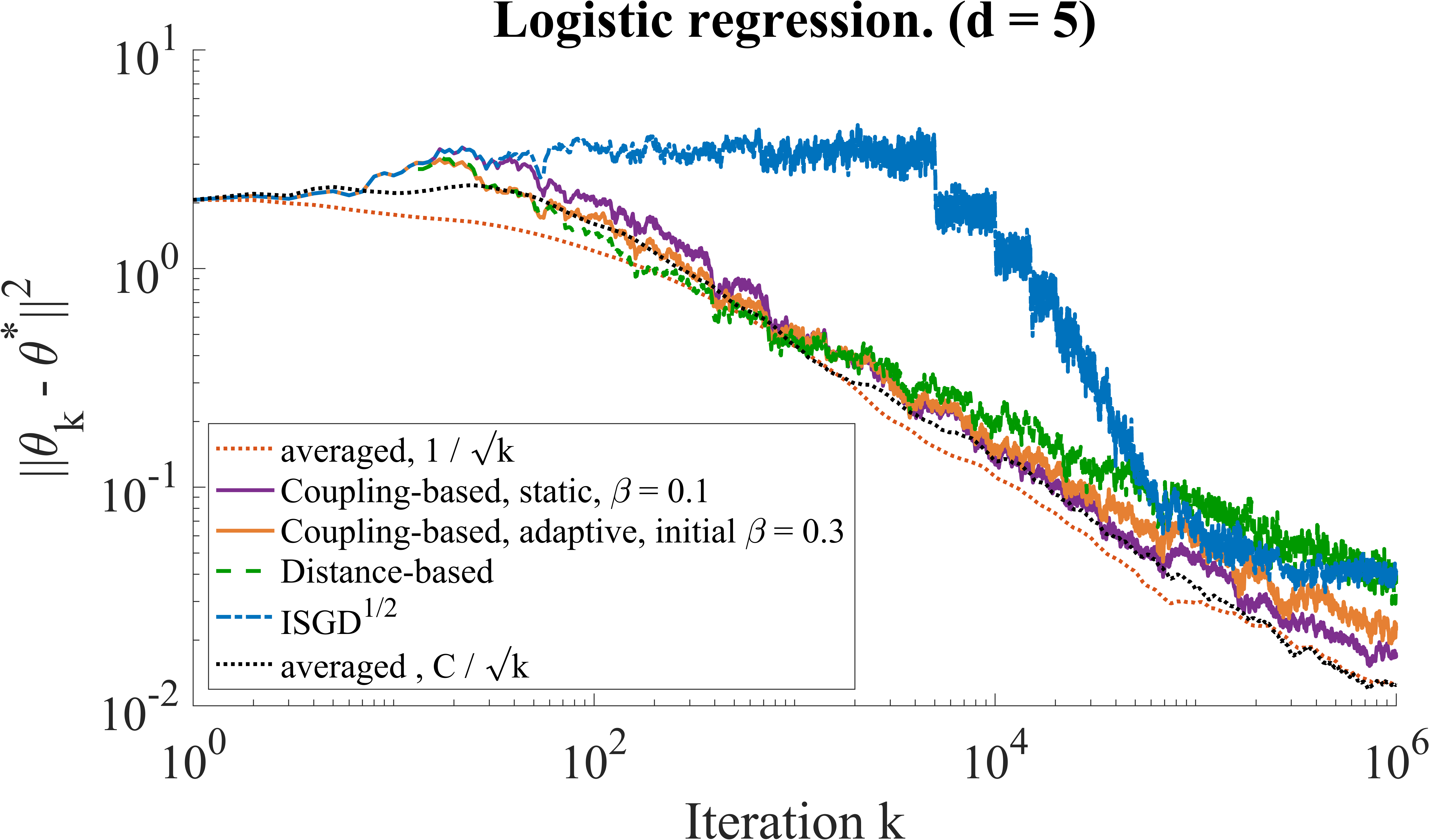}
    }
    \subfigure{
        \includegraphics[width=0.41\textwidth]{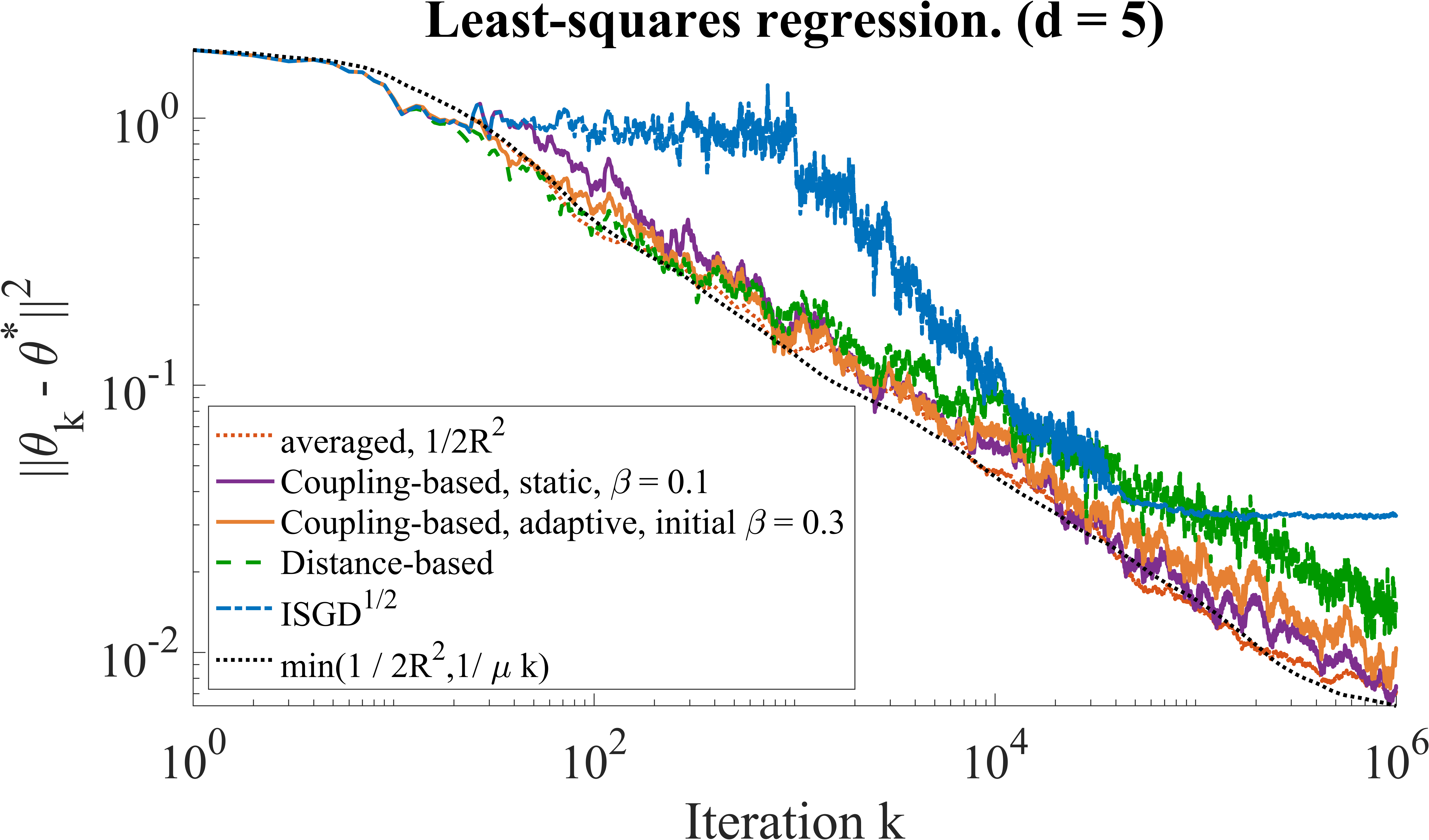}
    }

    \caption{Logistic regression (left) and Least squares regression (right).
    The initial stepsize of coupling/distance-based and $\text{ISGD}^{1/2}$ is $\gamma_0 = 4/R^2$ for logistic regression, and $\gamma_0 = 1/2R^2$ for least squares. The errors are averaged over $10$ replications.}
    \label{fig: Logistic_least_compare}
\end{figure*}

\begin{figure*}[htb]
    \centering
        \subfigure[Logistic regression: different stepsize decay factor $r$]{
        \includegraphics[width=0.41\textwidth]{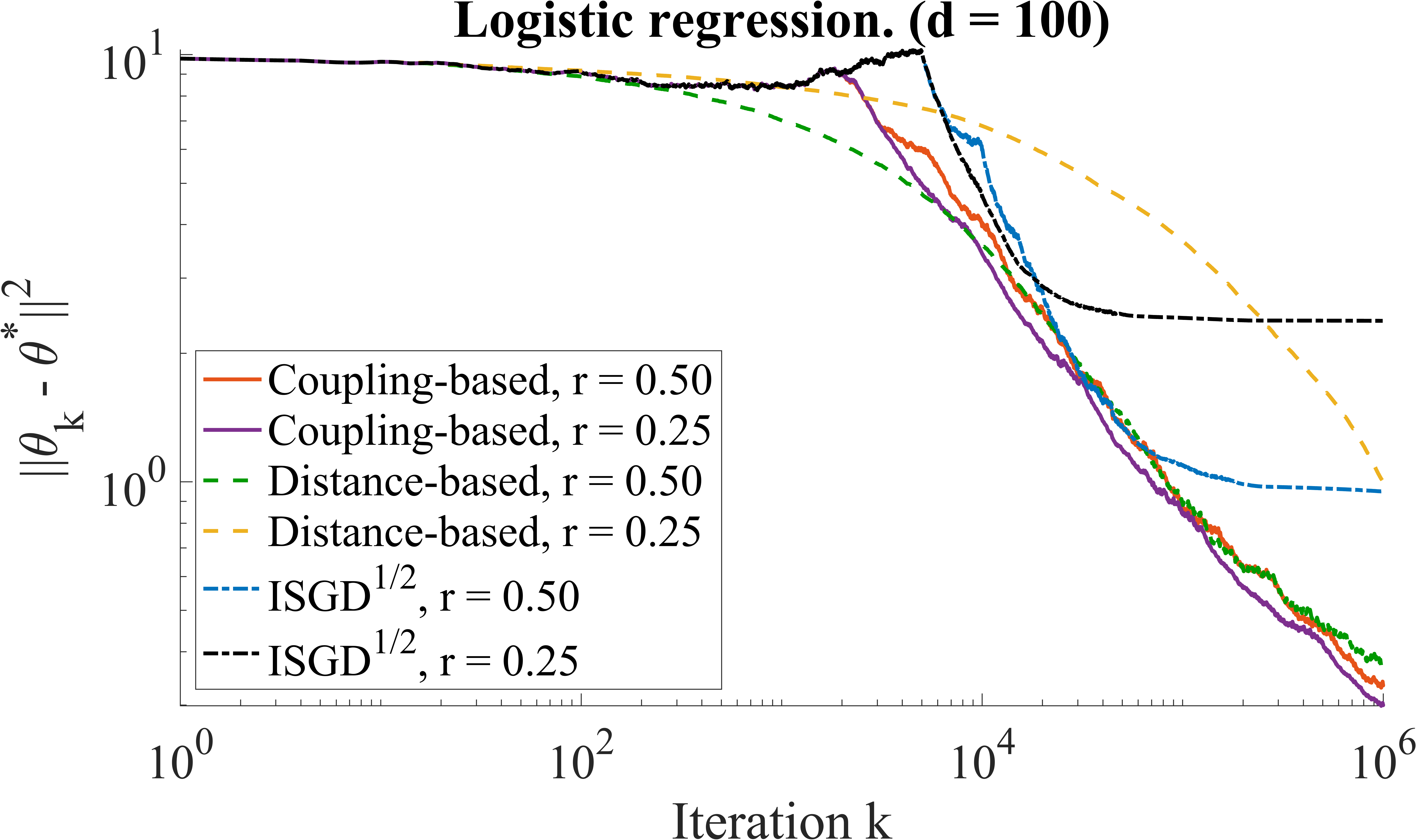} \label{fig: Robustness_LR_r}
    }
    \subfigure[Logistic regression: different initial threshold $\beta$]{
        \includegraphics[width=0.41\textwidth]{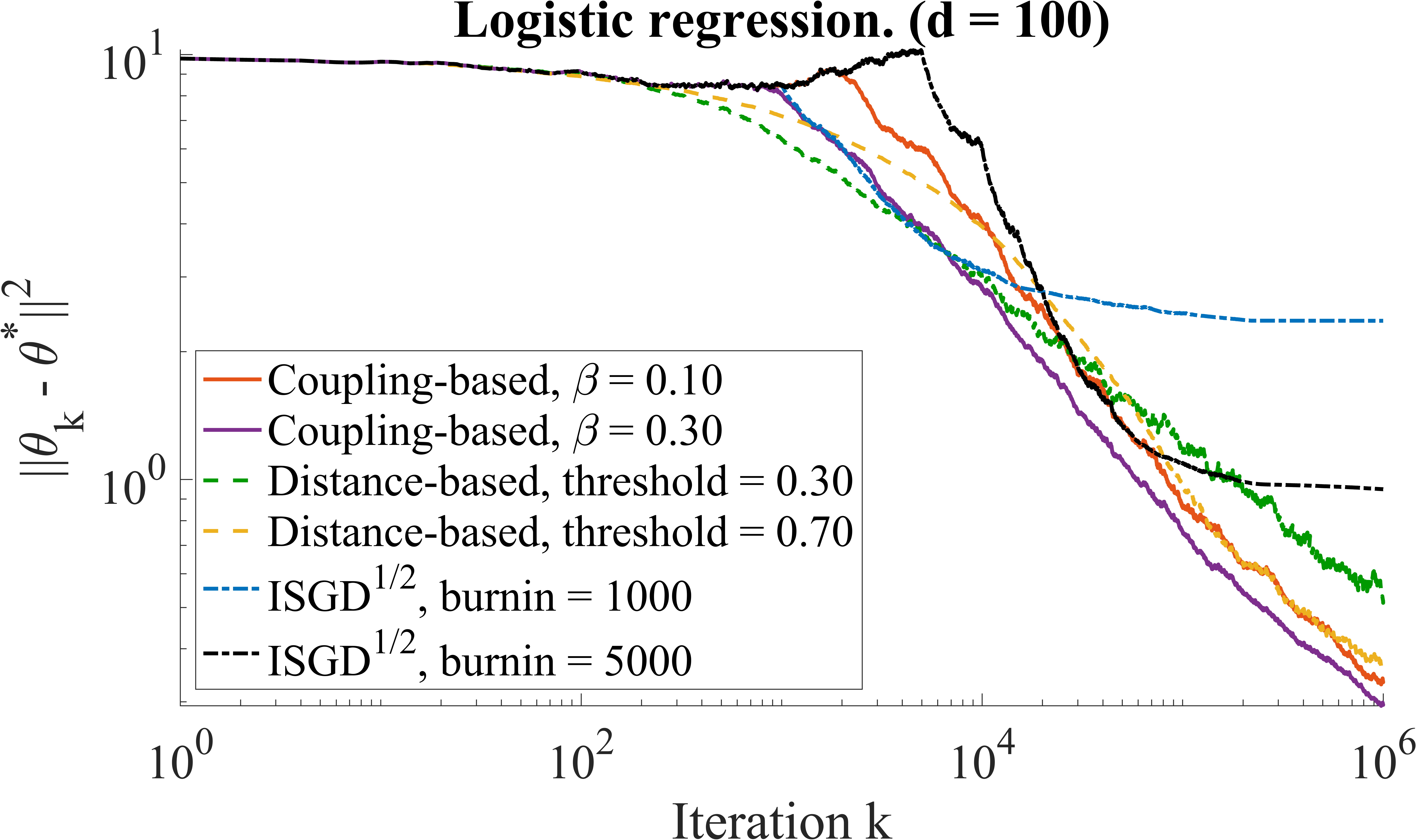} \label{fig: Robstuness_LR_thresh}
    }
    \subfigure[LSR: different stepsize decay factor $r$]{
        \includegraphics[width=0.41\textwidth]{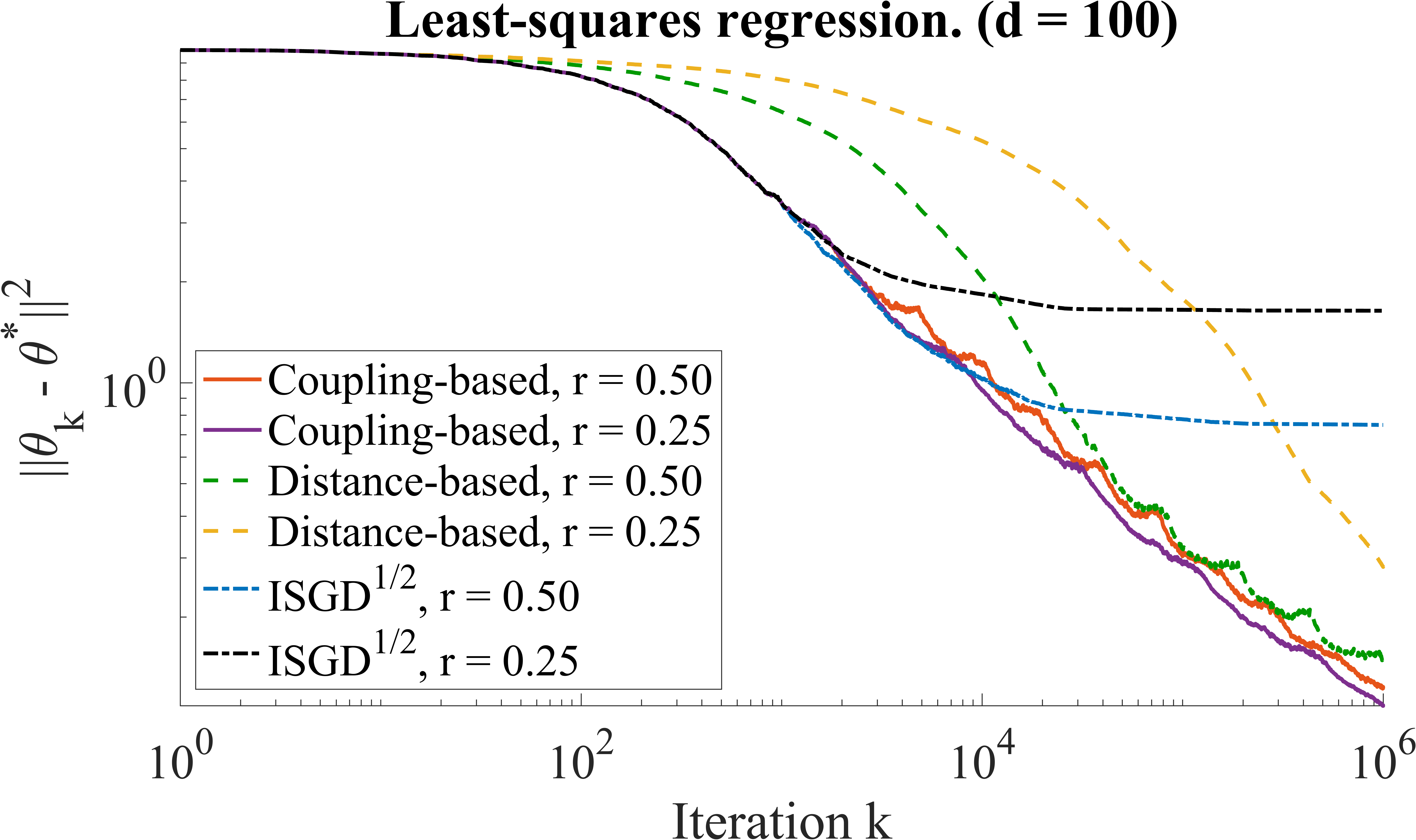}\label{fig: Robustness_LSR_r}
    }
    \subfigure[LSR: different initial threshold $\beta$]{
        \includegraphics[width=0.41\textwidth]{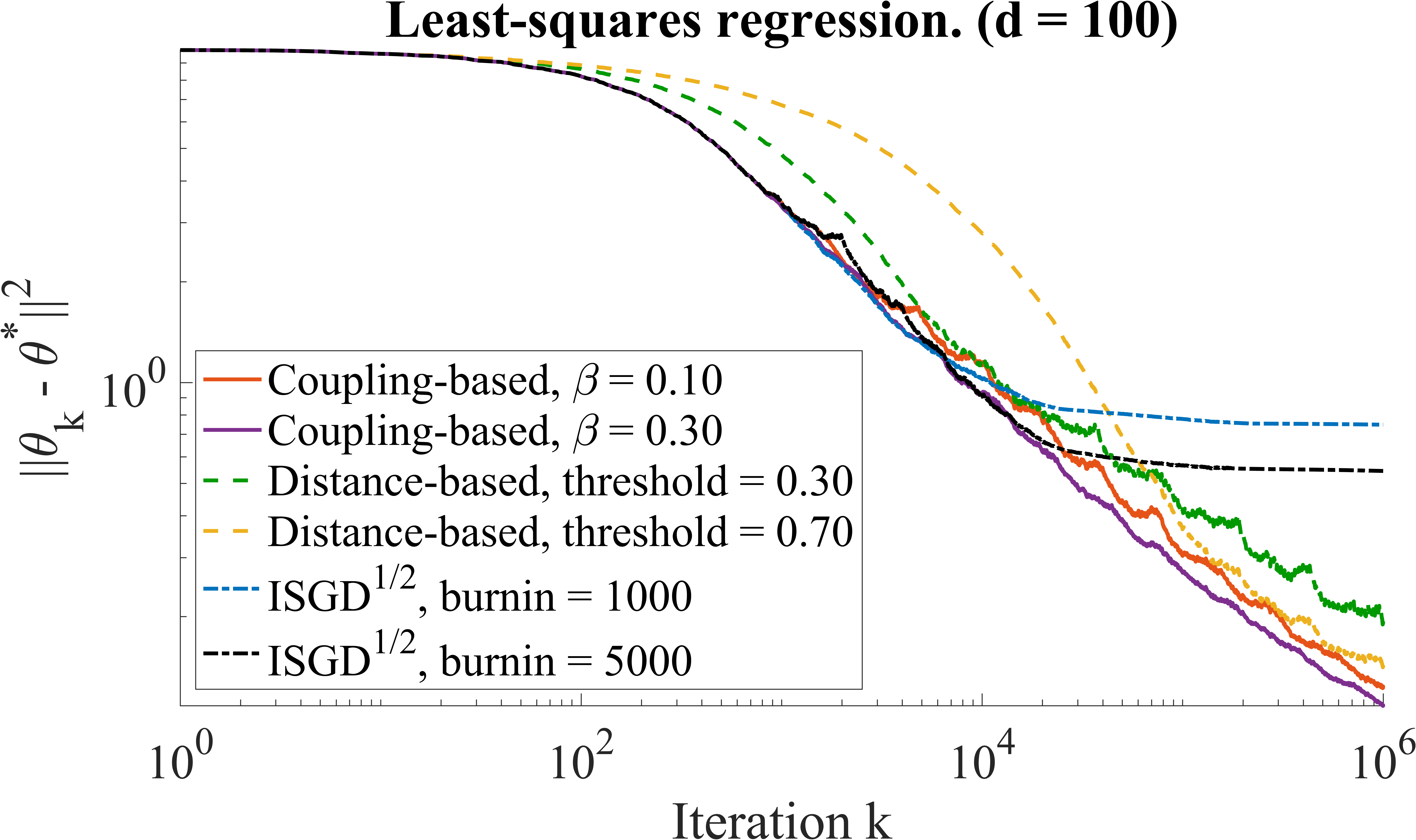} \label{fig: Robustness_LSR_thresh}
    }
    \caption{Robustness results under logistic regression and least squares regression (LSR) with $d = 100$. }
    \label{fig: Robustness}
\end{figure*}

\noindent\textbf{Logistic Regression.}
The objective function $f$ is defined as $f(\theta)=\mathbb{E}[\log(1+e^{-y_i \langle x_i, \theta \rangle})]$. The inputs $x_i$ are i.i.d. drawn from a multivariate normal distribution $\mathcal{N}(0, H)$, where $H$ is a diagonal matrix with dimension $d$.
The outputs $y_i\in\{-1, 1\}$ are generated according to the logistic probabilistic model.
We include averaged-SGD with stepsizes $\gamma_k = {1}/{\sqrt{k}}$ as a baseline, which achieves the optimal rate of $\mathcal{O}\left({1}/{n}\right)$ \cite{bach2014adaptivity}. In addition, we evaluate averaged-SGD with stepsizes $\gamma_k = C/\sqrt{k}$, tuning the parameter $C$ to achieve the best performance.

\noindent\textbf{Least Squares Regression.}
The objective function $f$ is given by $f(\theta) = \frac{1}{2}\mathbb{E}[(y_i - \langle x_i, \theta \rangle)^2]$. 
The inputs $x_i$ are the same as in the logistic regression model. 
The outputs $y_i$ are generated according to: $y_i = \langle x_i, \theta^\star \rangle + \varepsilon_i,$ where $\varepsilon_i$ are i.i.d. noise following a normal distribution $\mathcal{N}(0, \sigma^2)$. 
We include averaged-SGD with a constant stepsize $\gamma = 1/2R^2$, where $R^2 = \text{Tr} H$, due to its optimal convergence rate $\mathcal{O}\left({\sigma^2 d}/{n}\right)$ \cite{bach2013non}. We also evaluate SGD with stepsize $\gamma_k=1/{\mu k},$ which achieves a rate of $1/{\mu k}.$

\noindent\textbf{ResNet on CIFAR-10.} 
We consider an 18-layer ResNet model \cite{he2016deep} and train it on the CIFAR-10 dataset \cite{krizhevsky2009learning}. The training is conducted using an initial stepsize of 0.01, and a batch size of 128. To employ our coupling-based algorithm, we use PyTorch’s CustomScheduler() scheduler, with a stepsize decay factor of 0.1, a patience parameter of 10, and a threshold of 0.95. Specifically, we initialize two different ResNet models, and monitor the distance between their parameters. When the distance falls below the threshold for a specific number of epochs, we reduce the stepsize by the decay factor.

\subsection{Main Results}\label{sec: Comparison Results}

Our first set of experiments aim to demonstrate the efficiency of our coupling-based diagnostic for detecting stationarity of the iterates. 
We compare our algorithms against SGD with different constant stepsizes $\gamma$ for the least squares regression setting. The results are presented in Figure \ref{fig:constant_step}. Note that a larger constant stepsize leads to a faster convergence, but inducing a larger error when reaching saturation. We observe that, for both the static (Algorithm \ref{alg:static}) and adaptive coupling-based methods (Algorithm \ref{alg:adaptive}), the restarts (solid vertical line) mostly happen when the constant stepsize iterates start reaching saturation. That is, our coupling-based diagnostic statistic almost perfectly detects the stationarity of the iterates, without early on restart nor delay. Compared with the static variant (Figure \ref{fig:constant_step} left), the adaptive version (Figure \ref{fig:constant_step} right) becomes more conservative with the decreasing of threshold $\beta,$ leading to more accurate detection of stationarity; see, for example, the 5th restart. 

We next compare our method to other algorithms under logistic regression and least squares regression. For both settings, we consider synthetic dataset with size $n=1e6$ and dimension $d = (5, 20, 50, 100)$. The results of $d=5$ are presented in Figure \ref{fig: Logistic_least_compare}, and results of other dimensions are provided in Appendix \ref{App: Additional_Experiments}. We observe that both static and adaptive variant of our algorithms consistently achieve superior performance across all settings, with more prominent advantages in higher-dimensional cases. 
In particular, for logistic regression, our methods achieve comparable performance as the averaged-SGD with the best tuned diminishing stepsizes $\gamma_k=C/\sqrt{k}.$ Under the least squares regression, the performance of our methods is close to that with stepsizes $1/{\mu k}$, but without knowing the problem parameter $\mu.$ We note that the Pflug's diagnostic-based method $\text{ISGD}^{1/2}$ (blue lines) prematurely detects convergence and restarts too often. Consequently, it results in small stepsizes quickly and thus stopping further progress towards the optimal point early on.

\noindent\textbf{ResNet.} To demonstrate the general applicability of our coupling-based algorithm, we applied it to deep learning tasks. As shown in Figure \ref{fig:ResNet18}, our method (blue line) achieved a $94\%$ test accuracy, comparable to state-of-the-art results for ResNet18 on this benchmark. Although the test accuracy with constant stepsize (red line) appears to plateau around epoch 75, manually adjusting the learning rate at this point (green line) resulted in trivial improvements. This suggests that the adjustment was likely made too early. In contrast, our algorithm effectively identifies the optimal adjustment point at epoch 145, ensuring meaningful improvements.

\begin{figure}[htbp]
    \centering
\includegraphics[width=0.45\textwidth]{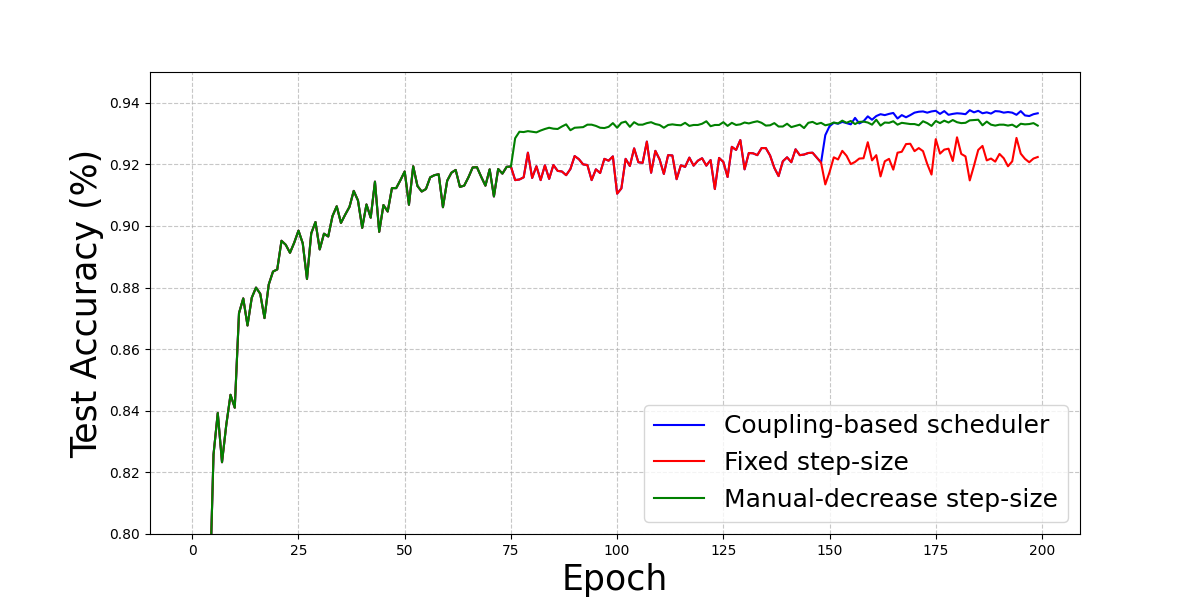}
    \caption{ResNet-18 Test Accuracy on CIFAR-10 dataset. The initial stepsizes are 0.01.}
    \label{fig:ResNet18}
\end{figure}

\subsection{Robustness Study}\label{sec: Robustness Results}

Our last set of experiments aim to study the sensitivity of our methods to algorithm parameters, including the stepsize decay factor $r$, the threshold $\beta$ and backward steps $b$ for re-initialization of auxiliary sequence. Here we focus on our static coupling-based algorithm for logistic regression and least squares regression with dimension $d = 100$. Additional robustness results for adaptive variant and other settings are provided in Appendix \ref{App: Additional_Experiments}.

The results are presented in Figure \ref{fig: Robustness}, where each row corresponds to different data models, and each column corresponds to variations of different hyper-parameters. Specifically, 
we first perform the experiments with different stepsize decay factor $r$, while keeping other parameters fixed. As illustrated in Figure \ref{fig: Robustness_LR_r} and \ref{fig: Robustness_LSR_r}, our algorithm is insensitive to the stepsize decay factor $r$, and consistently achieves the best performance for both logistic regression and least squares regression. In contrast, the performance of the distance-based algorithm and the $\text{ISGD}^{1/2}$ algorithm varies significantly with $r$.  
We also test the sensitivity of our method and the distance-based algorithm with respect to the threshold parameter for convergence diagnostic, as well as the $\text{ISGD}^{1/2}$ algorithm w.r.t.\ the burn-in parameter. Figure \ref{fig: Robstuness_LR_thresh} and \ref{fig: Robustness_LSR_thresh} show that our coupling-based method achieves similar superior performance with different $\beta$ across various settings.

\section{Conclusion and Future Work} \label{sec:conclusion}
In this paper, we study the convergence diagnostic of stochastic gradient descent with constant stepsize. We propose a novel diagnostic statistic based on the distance between two coupled SGD sequences. The distance is easy to compute and is shown to track the distributional convergence of SGD iterates under constant stepsize theoretically. A key advantage of our method is that the two coupled iterates can be run in parallel with the same data, which avoids an increase in the wall-clock time despite the additional computational cost.

We conduct an extensive range of experiments to numerically examine the properties of our proposed coupling-based methods. Our algorithms consistently outperform existing approaches. Moreover, our method can be applied to other stochastic optimization frameworks beyond SGD. For applications like reinforcement learning, where data is often limited, our adaptive stepsize approach can make more efficient use of data and accelerate convergence.

There are several research directions one can take to extend our work. While our empirical results demonstrate robustness of our approach to the threshold for convergence diagnostic, it would be valuable to develop a principled rule for choosing this parameter. Another future direction worth pursuing is to extend our analysis for other problem classes, particularly non-convex settings that satisfy structural conditions such as dissipativity or the generalized Polyak-Lojasiewicz condition, and general stochastic approximation frameworks.

\section*{Acknowledgments}
This project is supported in part by National Science Foundation (NSF) grants CNS-1955997,  EPCN-2432546 and CAREER Award 2339794.

\bibliography{aaai25}

\clearpage
\appendix

\section{Proofs}\label{App: Proofs}
In this section, we provide missing proof of results presented in the main paper.

\subsection{A Technical Lemma}
\label{sec:tech_lemma}

We state and prove a simple technical lemma, which is used in our analysis for the general convex setting (Theorem \ref{thm: proximity-general}). 

\begin{lemma}\label{lemma: 1} If \( k_0 = \frac{4L}{\mu} \), then
\[
(1 - \gamma\mu)^{k_0} \leq 1 - 2\gamma L + \gamma^2 \mu, \quad \forall\gamma \in [0, \gamma_0],
\]
where \( \gamma_0 = \min \left\{ \frac{1}{4L}, \frac{2L}{\mu} \right\} \).
\end{lemma}

\begin{proof}
    By inequality \( e^z \geq 1 + z \) for all \( z \in \mathbb{R} \), we have
    \[
    (1 - \gamma\mu)^{k_0} \leq e^{-{k_0}\gamma\mu}.
    \]
    Note that \( {k_0}\gamma\mu =  4L\gamma \leq 1 \). Using the inequality \( e^{-x} \leq 1 - \dfrac{x}{2} \) for \( 0 \leq x \leq 1 \), we have
    \[
    e^{-{k_0}\gamma\mu} \leq 1 - \dfrac{{k_0}\gamma\mu}{2} = 1 - 2\gamma L \leq 1 - 2\gamma L + \gamma^2 \mu.
    \]
    Combining the above two inequalities completes the proof.
\end{proof}

\subsection{Proof of Proposition \ref{proposition: Dk_quadratic}}

\begin{proof}
Define $D_k := \theta^{(1)}_k - \theta^{(2)}_k$, we have
\begin{align*}
D_k &= \theta^{(1)}_k - \theta^{(2)}_k\\
&= \theta^{(1)}_{k-1} - \gamma(H\theta^{(1)}_{k-1}-\xi_k)\\
&\quad - [\theta^{(2)}_{k-1} - \gamma(H\theta^{(2)}_{k-1}-\xi_k)]\quad\text{noise is identical}\\
&= \theta^{(1)}_{k-1} - \gamma H\theta^{(1)}_{k-1} - (\theta^{(2)}_{k-1} - \gamma H\theta^{(2)}_{k-1})\\
&= (I-\gamma H)(\theta^{(1)}_{k-1} - \theta^{(2)}_{k-1})\\
&= (I-\gamma H)D_{k-1}
\end{align*}
By straightforward induction, we can derive the following relationships:
\begin{equation*}
D_k = (I-\gamma H)^kD_{0}
\end{equation*}
Therefore, taking the expectation gives
\begin{align*}
    \mathbb{E}[\|D_k\|^2] &= \mathbb{E}[\|(I-\gamma H)^kD_0\|^2]\\
    &= \mathbb{E}[D_0^\top(I-\gamma H)^{2k}D_0].
\end{align*}

By Assumption \ref{assumption: bounded_variance}, the smallest eigenvalue of $H$ is $\mu.$ With $\gamma\in (0,1/L),$ we have 
\begin{align*}
    \mathbb{E}[\|D_k\|^2] &\leq (1-\gamma \mu)^{2k} \mathbb{E}[\|D_0\|^2].
\end{align*}
\end{proof}

\subsection{Proof of Claim \ref{claim:W2_quadratic}}

\begin{proof}
    The proof follows from the proof of Proposition 2 in \citet{dieuleveut2020bridging}, by using the fact that the objective function $f$ is quadratic. 
    Let $\gamma\in(0,2/L)$ and $\nu_{1},\nu_{2}\in\CP_{2}(\R^{d})$.
By \citet[Theorem   4.1]{villani2009optimal}, there exists a couple of random
variables $\theta_{0}^{(1)},\theta_{0}^{(2)}$ such that $W_{2}^{2}(\nu_{1},\nu_{2})=\E\left[\norm{\theta_{0}^{(1)}-\theta_{0}^{(2)}}^{2}\right]$
independent of $(\epsilon_{k})_{k\in\N}$. Let $(\theta_{k}^{(1)})_{k\geq0}$,$(\theta_{k}^{(2)})_{k\geq0}$
be the SGD iterates with the same constant stepsize $\gamma$, and
sharing the same noise, but starting from $\theta_{0}^{(1)}$ and
$\theta_{0}^{(2)}$ respectively . That is, for all $k\geq0$, 
\begin{equation}
\begin{cases}
\theta_{k+1}^{(1)} & =\theta_{k}^{(1)}-\gamma\big[f'(\theta_{k}^{(1)})+\varepsilon_{k}\big]\\
\theta_{k+1}^{(2)} & =\theta_{k}^{(2)}-\gamma\big[f'(\theta_{k}^{(2)})+\varepsilon_{k}\big].
\end{cases}\label{eq:def_coupling}
\end{equation}
Since $\theta_{0}^{(1)},\theta_{0}^{(2)}$ are independent of $\varepsilon_{1}$,
we have for $i,j\in\{1,2\}$, that 
\begin{equation}
\E[\langle\theta_{0}^{(i)},\varepsilon(\theta_{0}^{(j)})\rangle]=0.\label{eq:indep_noise_initial_cond}
\end{equation}
We slightly overload the notation, letting $P_{\gamma}^{k}(\nu,\cdot):=\int_{\R^{d}}\nu(d\theta_{0})P_{\gamma}^{k}(\theta_{0},\cdot)$
denote the distribution of the $k$-th iterate when initial $\theta_{0}\sim\nu$.
By the definition of the Wasserstein distance we get 
\begin{align*}
 & W_{2}^{2}\big(P_{\gamma}(\nu_{1},\cdot),P_{\gamma}(\nu_{2},\cdot)\big)\le\E\left[\|\theta_{1}^{(1)}-\theta_{2}^{(2)}\|^{2}\right]\\
 & =\E\left[\|\theta_{0}^{(1)}-\gamma f'(\theta_{0}^{(1)})-(\theta_{0}^{(2)}-\gamma f'(\theta_{0}^{(2)})))\|^{2}\right]\\
 & =\E\left[\|(I-\gamma H)\big(\theta_{0}^{(1)}-\theta_{0}^{(2)}\big)\|^{2}\right]\\
 & \leq(1-\gamma\lambda_{\min})^{2}\E\left[\norm{\theta_{0}^{(1)}-\theta_{0}^{(2)}}^{2}\right],
\end{align*}

By induction, setting $\rho:=(1-\gamma\lambda_{\min})^{2}$, we get,
\begin{align*}
W_{2}^{2}\big(P_{\gamma}^{k}(\nu_{1},\cdot),P_{\gamma}^{k}(\nu_{2},\cdot)\big) & \le\E\left[\|\theta_{k}^{(1)}-\theta_{k}^{(2)}\|^{2}\right]\\
 & \leq\rho\E\left[\|\theta_{k-1}^{(1)}-\theta_{k-1}^{(2)}\|^{2}\right]\\
 & \leq\rho^{k}W_{2}^{2}(\nu_{1},\nu_{2})
\end{align*}
For any given $\theta_{0}\in\R^{d},$ by taking $\nu_{1}=\delta_{\theta_{0}},$
and $\nu_{2}=\pi_{\gamma},$ and using the fact that the limit distribution
$\pi_{r}$ being invariant \cite[Proposition 2]{dieuleveut2020bridging}, we have
\begin{align*}
    W_{2}^{2}\big(P_{\gamma}^{k}(\theta_{0},\cdot),\pi_{\gamma}\big)&\leq\rho^{k}W_{2}^{2}(\delta_{\theta_{0}},\pi_{\gamma})\\
    &\leq\rho^{k}\E_{\theta\sim\pi_{\gamma}}\left[\norm{\theta_{0}-\theta}^{2}\right].
\end{align*}
\end{proof}

\section{Supplementary Experiments}\label{App: Additional_Experiments}
In this section, we present supplementary experiments and additional results to further support our main findings discussed in the paper.

\subsection{Additional results on Logistic regression and Least squares regression}

\begin{figure*}[htb]
    \centering
    \subfigure{
        \includegraphics[width=0.43\textwidth]{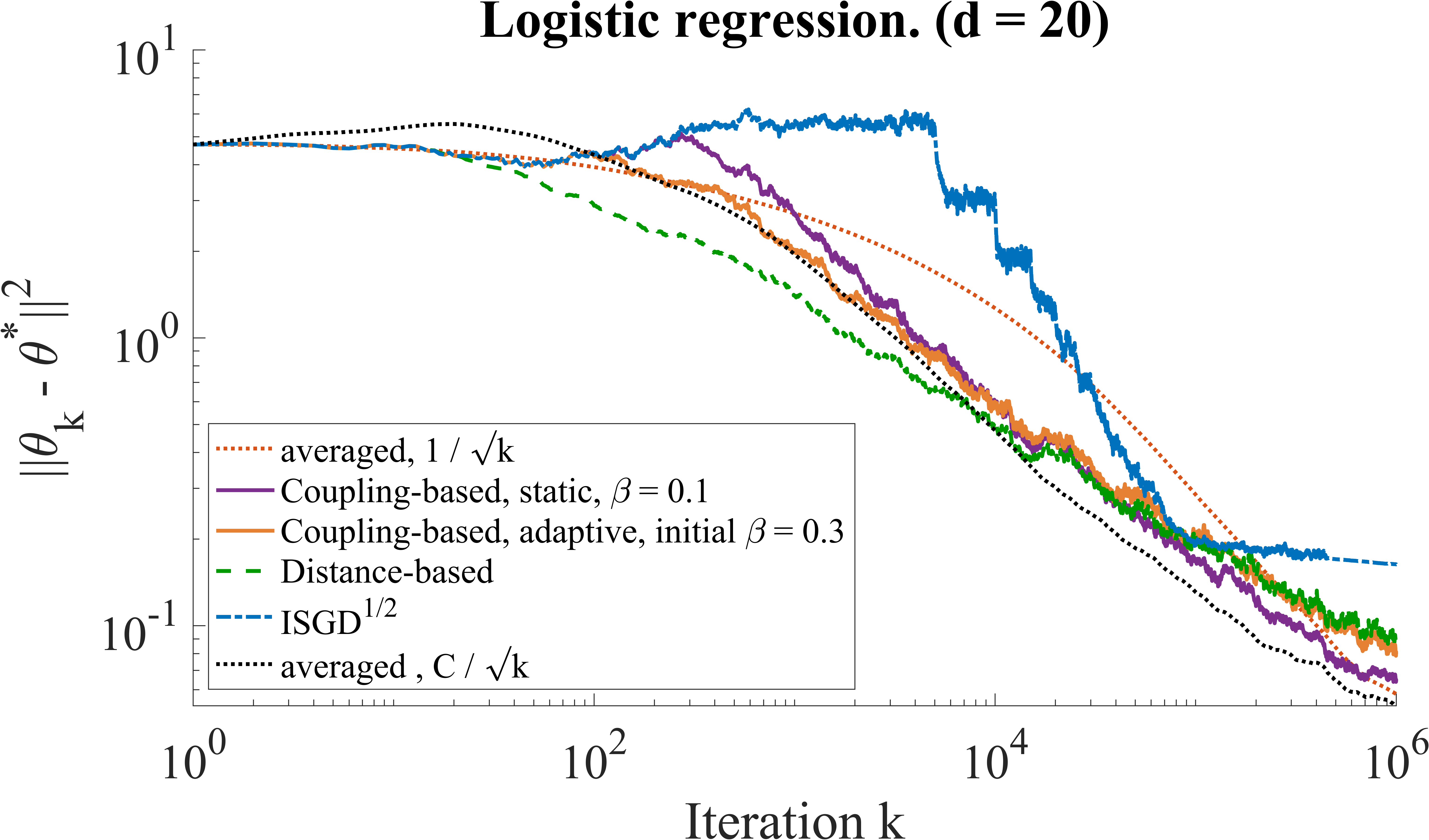}
    }
    \subfigure{
        \includegraphics[width=0.43\textwidth]{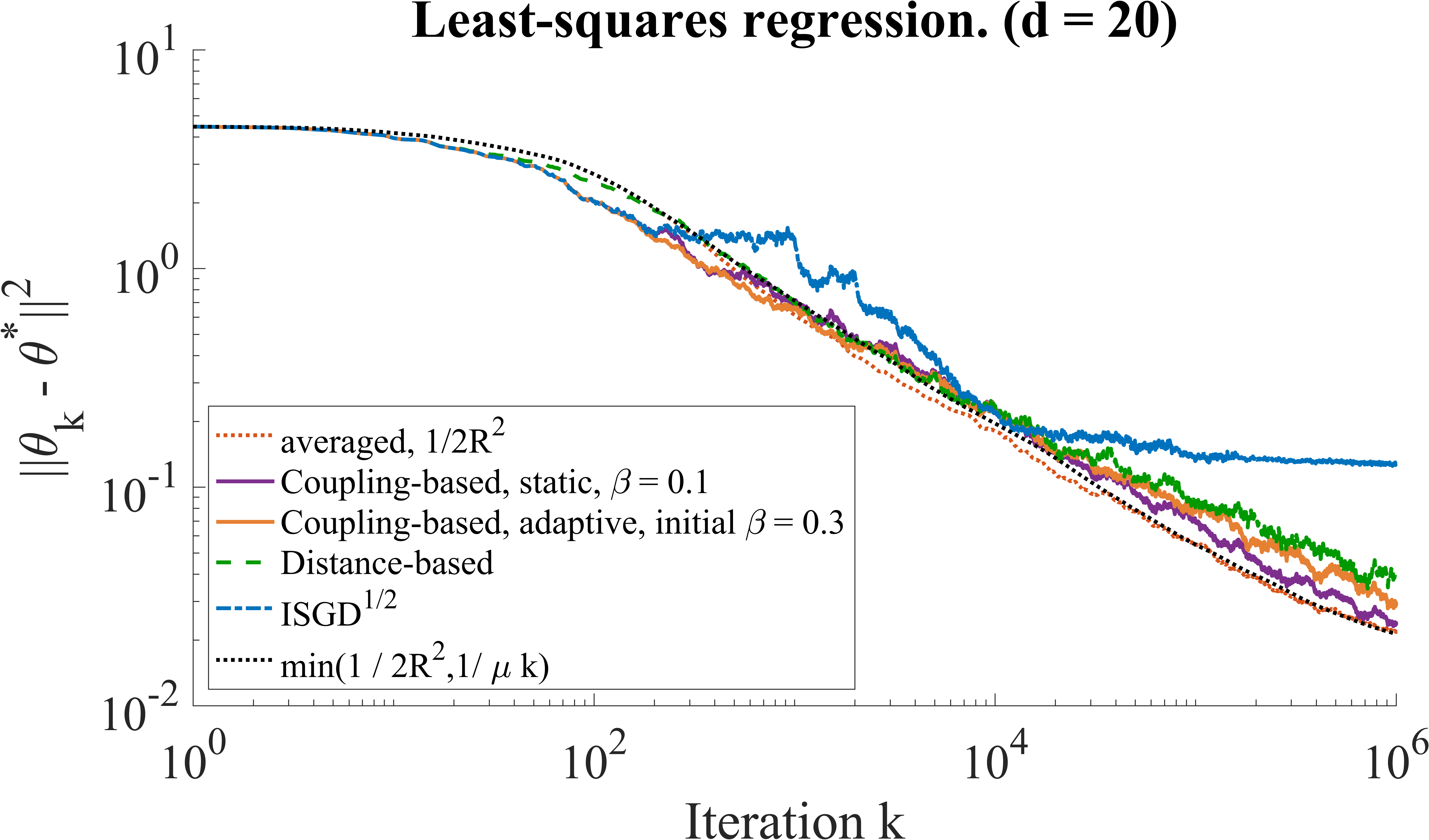}
    }
    \subfigure{
        \includegraphics[width=0.43\textwidth]{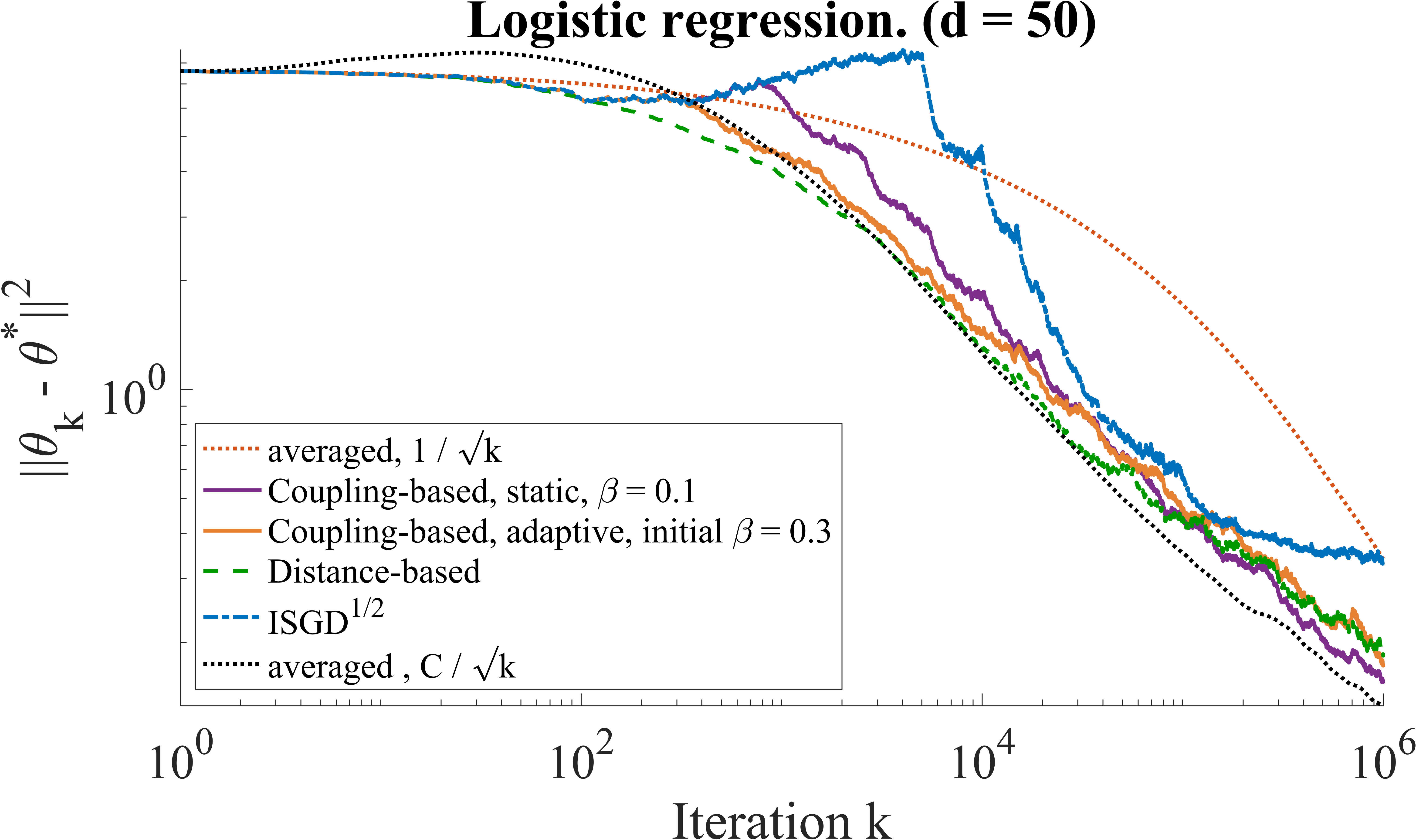}
    }
    \subfigure{
        \includegraphics[width=0.43\textwidth]{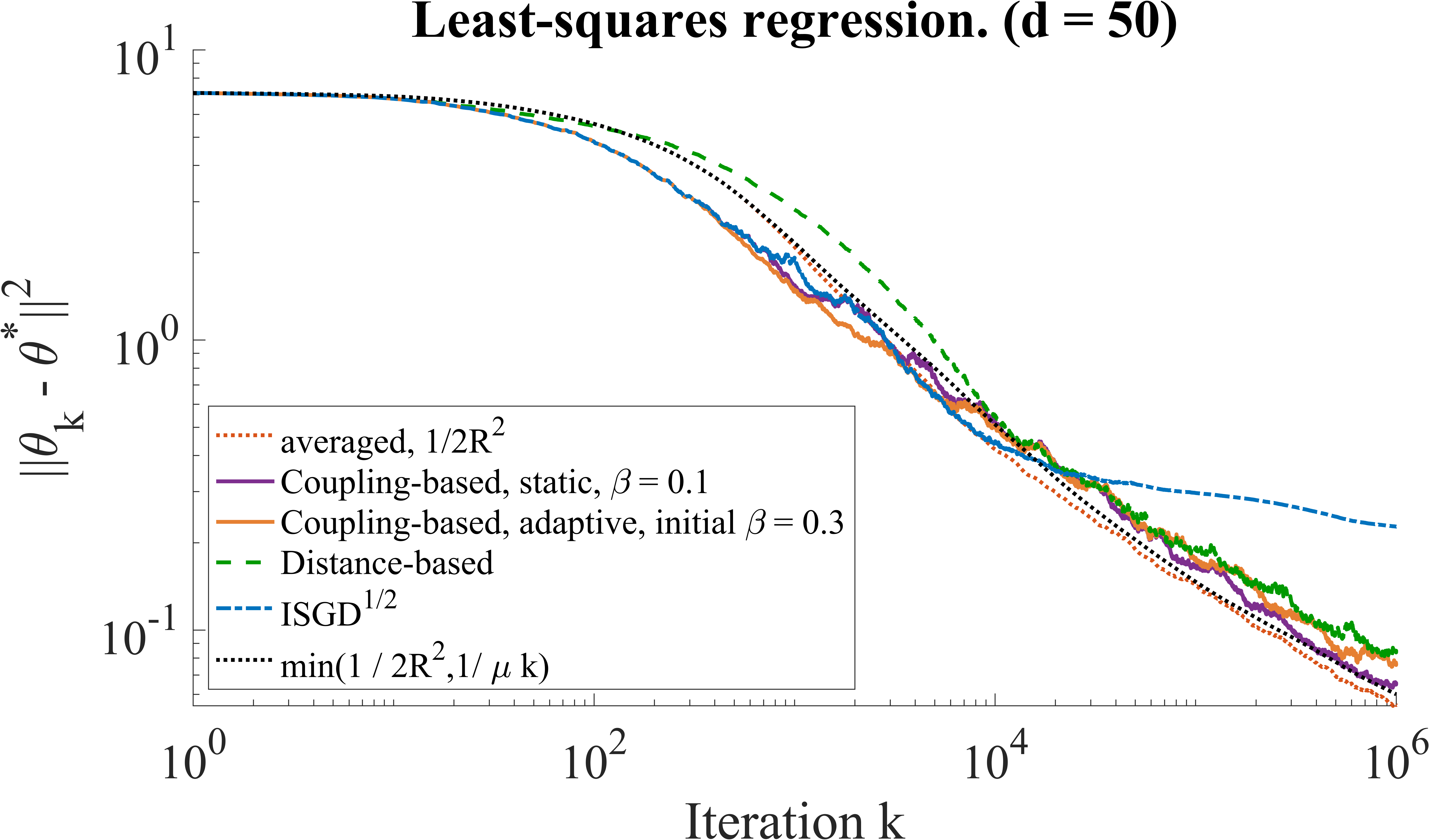}
    }
    \subfigure{
        \includegraphics[width=0.43\textwidth]{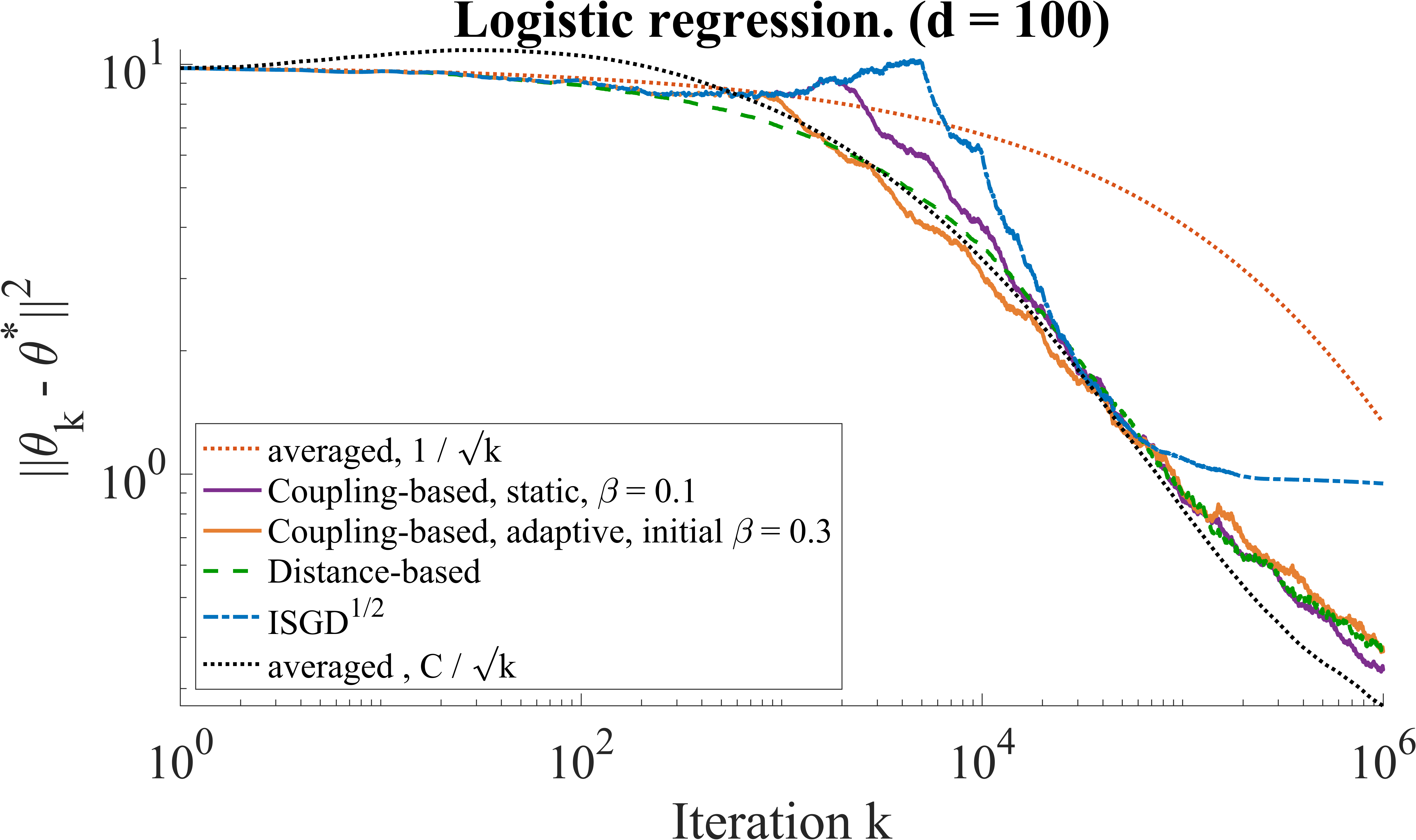}
    }
    \subfigure{
        \includegraphics[width=0.41\textwidth]{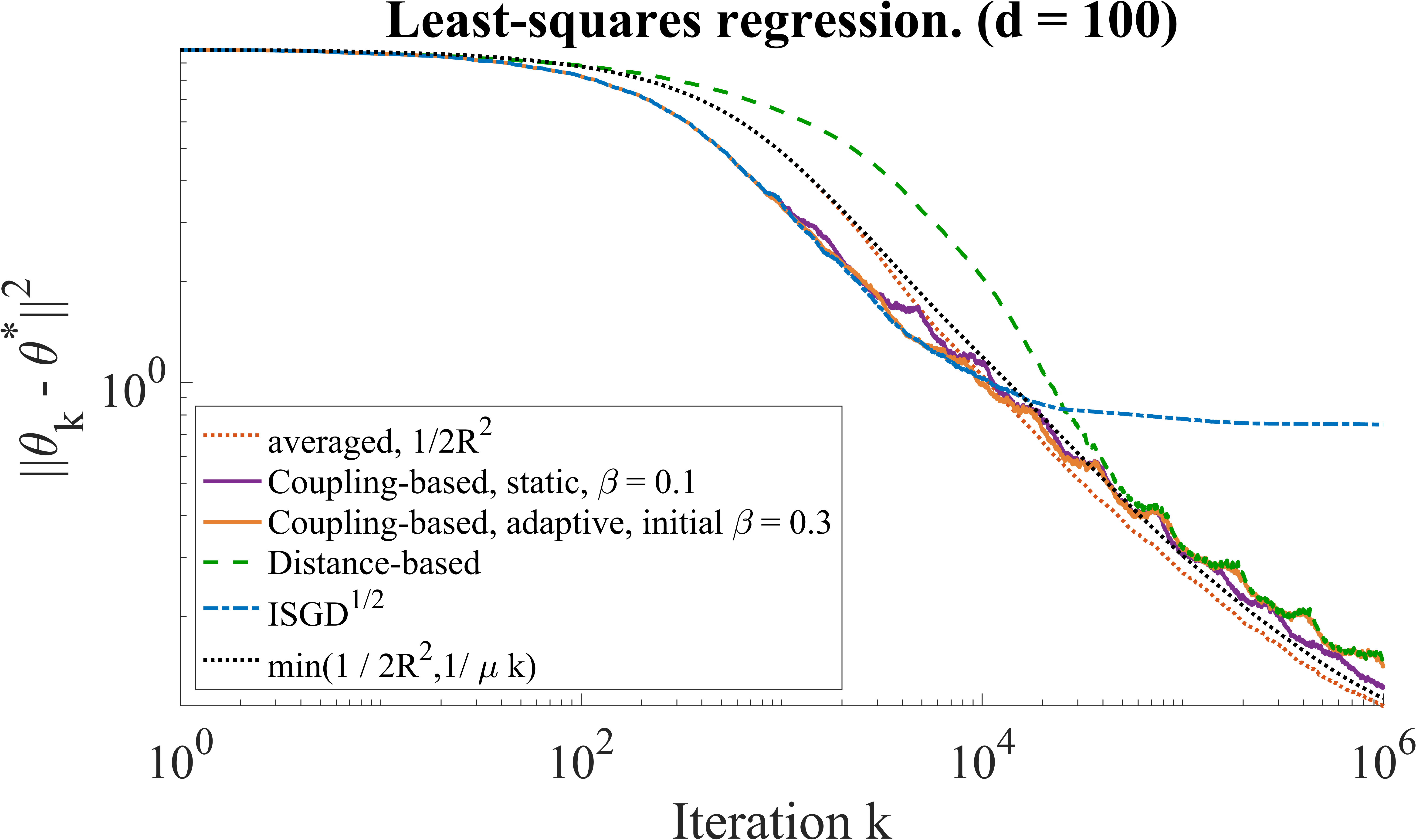}
    }
    \caption{Logistic regression (left) and Least squares regression (right) with different dimensions $d=(20,50,100)$.
    The initial stepsize of coupling/distance-based and $\text{ISGD}^{1/2}$ is $\gamma_0 = 4/R^2$ for logistic regression, and $\gamma_0 = 1/2R^2$ for least squares. The errors are averaged over $10$ replications.}
    \label{fig:LS_LSR_other_dim}
\end{figure*}

\begin{figure*}[htbp]
    \centering
    \subfigure[Logistic regression: different back steps $b$]{
    \includegraphics[width=0.43\textwidth]{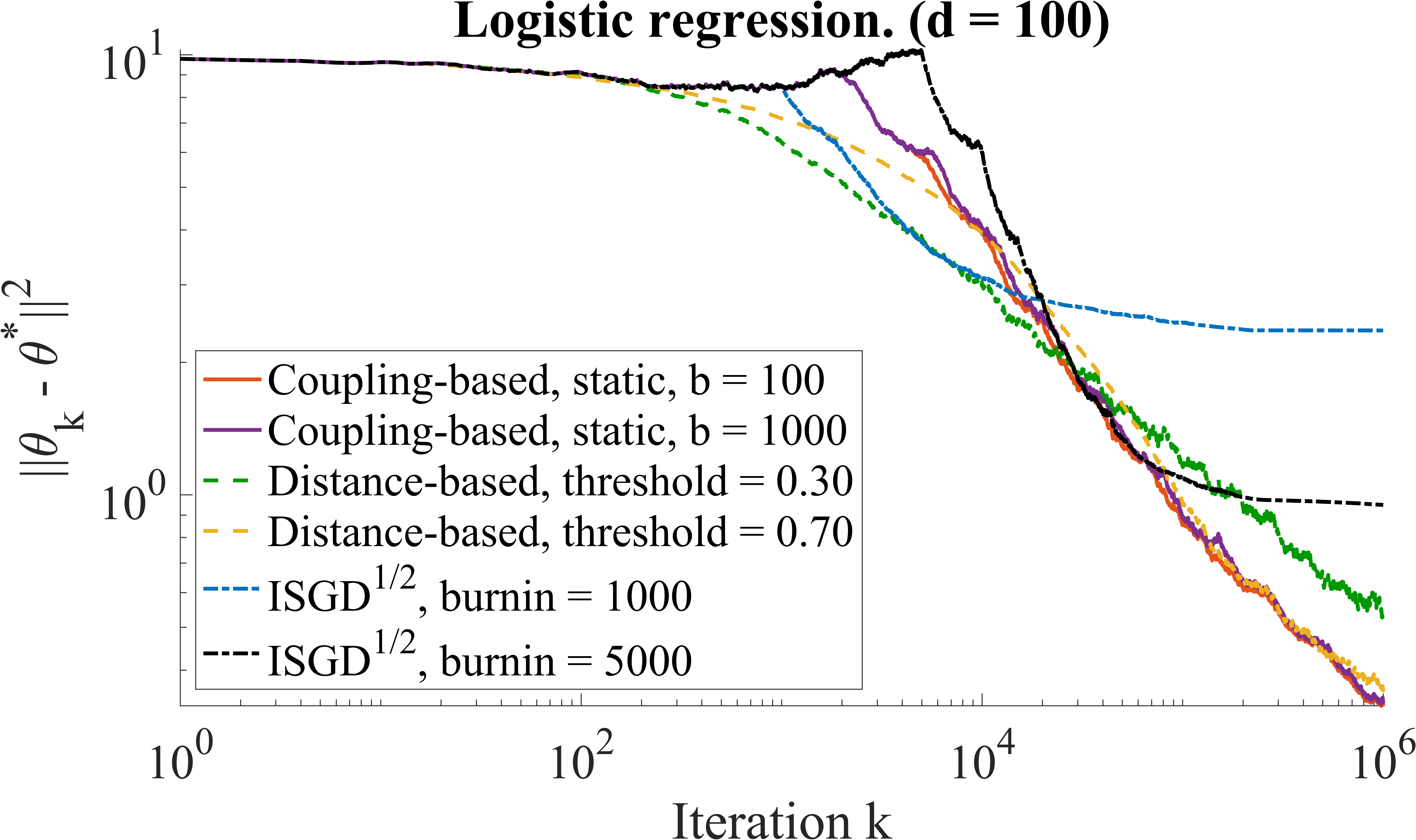}\label{fig: Robustness_LR_bstep_static}
    }
    \subfigure[LSR: different back steps $b$]{
    \includegraphics[width=0.43\textwidth]{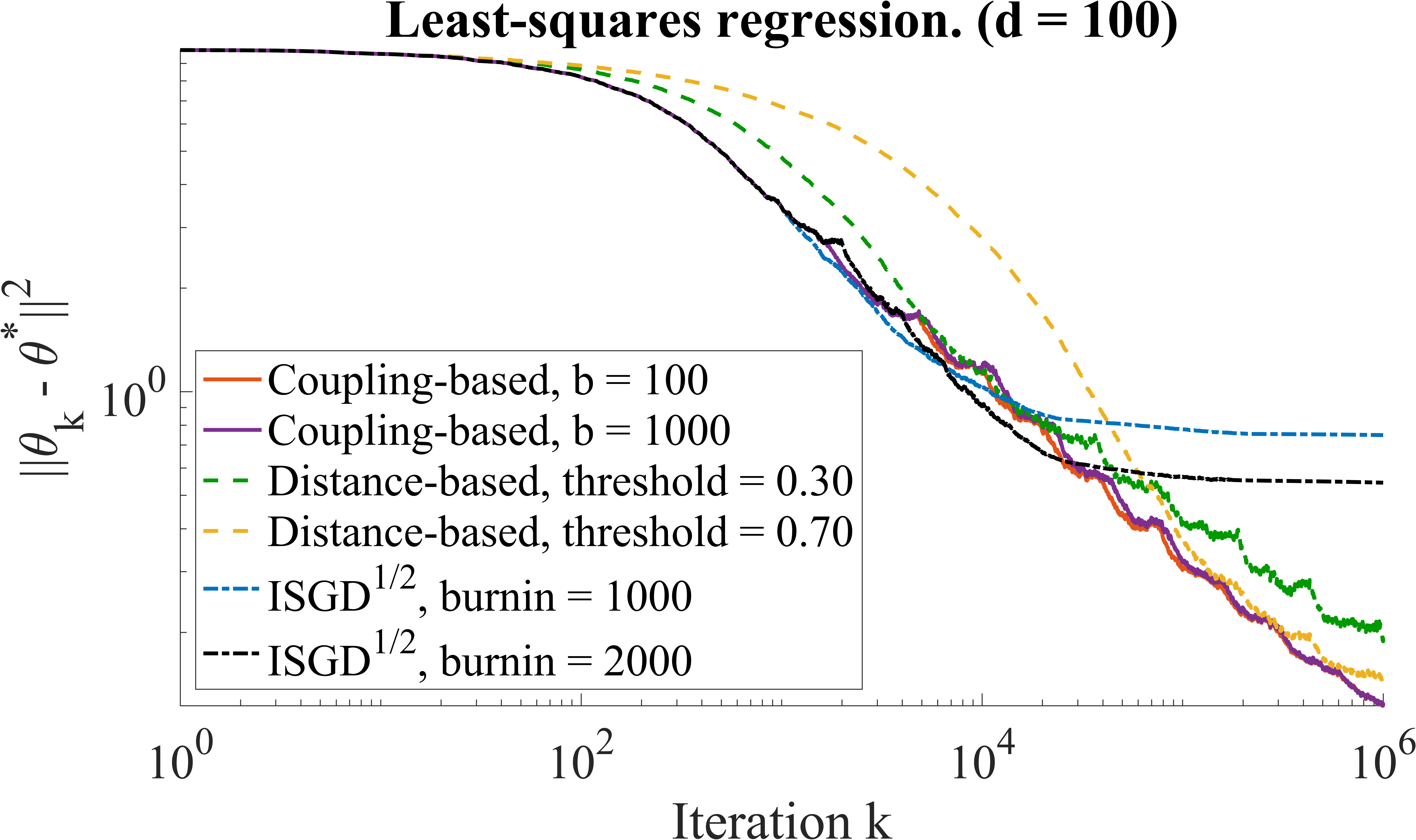}\label{fig: Robustness_LSR_bstep_static}
    }
    \caption{Robustness results under logistic regression and least squares regression (LSR) with $d = 100$ for Algorithm \ref{alg:static}.}
    \label{fig: Robustness_static_b}
\end{figure*}

We provide additional results on logistic regression and least squares regression settings, as described in Section \ref{sec: Experiments}, but with higher dimensions $d=(20,50,100)$. The results are presented in Figure \ref{fig:LS_LSR_other_dim}, where the left and right column corresponds to the logistic regression and the least square regression, respectively. We observe that our coupling-based algorithms with static/adaptive threshold achieves similar performance across all the settings. For logistic regression, the performance of our algorithms is comparable to that of the averaged SGD with the best-tuned decay stepsize $\gamma_k=C/\sqrt{k}$ across different dimensions. We note that the distance-based algorithm with best-tuned parameters converges slightly faster than our methods initially for logistic regression with $d=20,$ but incurring larger error at the end. The difference between our methods and the distance-based method becomes negligible for higher dimensional settings. For least squares regression, our algorithms consistently outperform other methods, and almost match the best performance of SGD with diminishing stepsize $\gamma_k=1/{\mu k}.$ The $\text{ISGD}^{1/2}$ algorithm leads to poor performance across all the settings, mainly due to the early on restart and frequent stepsize reduction. 

\subsection{Additional robustness results}

In Section \ref{sec: Robustness Results}, we have demonstrated the robustness of our coupling-based algorithm with static threshold, Algorithm \ref{alg:static}, to hyper-parameters including the stepsize decay factor $r$ and the initial threshold $\beta.$ We also study the sensitivity of our algorithm w.r.t.\ the backward steps parameter $b$ for the re-initialization of the auxiliary sequence. Figure \ref{fig: Robustness_static_b} shows that Algorithm \ref{alg:static} is robust to a wide range of the parameter $b$ values. 

We also investigate the sensitivity of our method with adaptive threshold, Algorithm \ref{alg:adaptive}, w.r.t.\ all the hyper-parameters: (1) the stepsize decay factor $r$; (2) the initial threshold $\beta$; (3)  the backward steps $b$; (4) the threshold decay factor $\eta$. We consider the same settings as that for Algorithm \ref{alg:static}, including logistic regression and least-squares regression with dimension $d=100.$ We report the results in Figure \ref{fig: Robustness_adaptive}, where each column corresponds to different setting, and each row corresponds to variations of different hyper-parameters (with other parameters fixed). From Figure \ref{fig: Robustness_LR_r_adaptive}--\ref{fig: Robustness_LSR_bstep_adaptive}, we observe that Algorithm \ref{alg:adaptive} enjoys similar robustness as Algorithm \ref{alg:static} w.r.t.\ a wide range values of the hyper-parameters $r$, initial threshold $\beta$ and $b.$ 
{For the threshold decay factor $\eta,$ Figure \ref{fig: Robustness_LR_eta}--\ref{fig: Robustness_LSR_eta} show that a small $\eta$ leads to slightly worse performance for both settings. We note that a small $\eta$ indicates a large reduction of the threshold at each restart, which would result in a more conservative convergence diagnostic. Our experiment results imply that choosing a slightly larger $\eta$, for instance, $\eta=0.75$ would yield robust performance. }

\begin{figure*}[htbp]
    \centering
        \subfigure[Logistic regression: different stepsize decay factor $r$]{
        \includegraphics[width=0.43\textwidth]{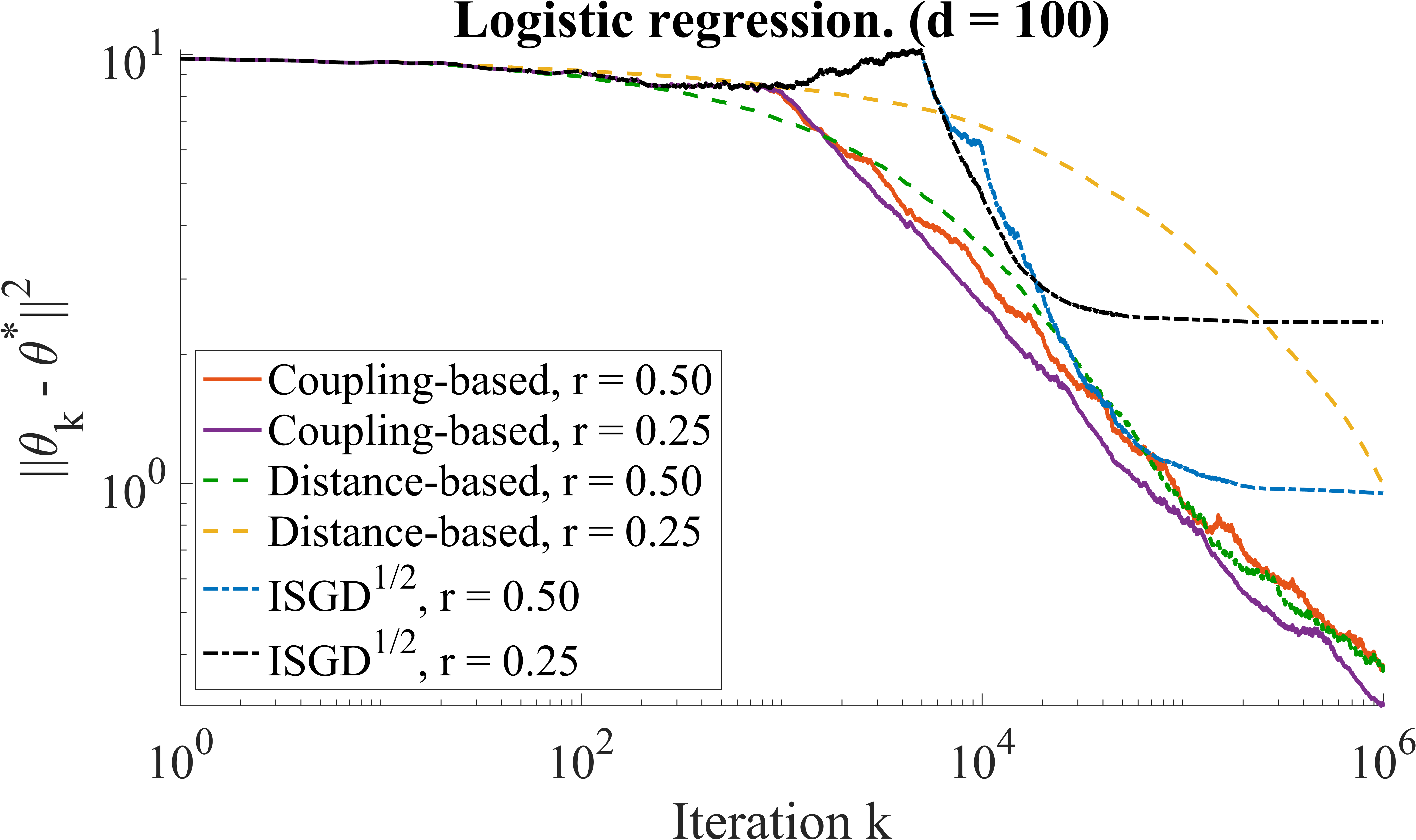} \label{fig: Robustness_LR_r_adaptive}
    }
    \subfigure[LSR: different stepsize decay factor $r$]{
        \includegraphics[width=0.43\textwidth]{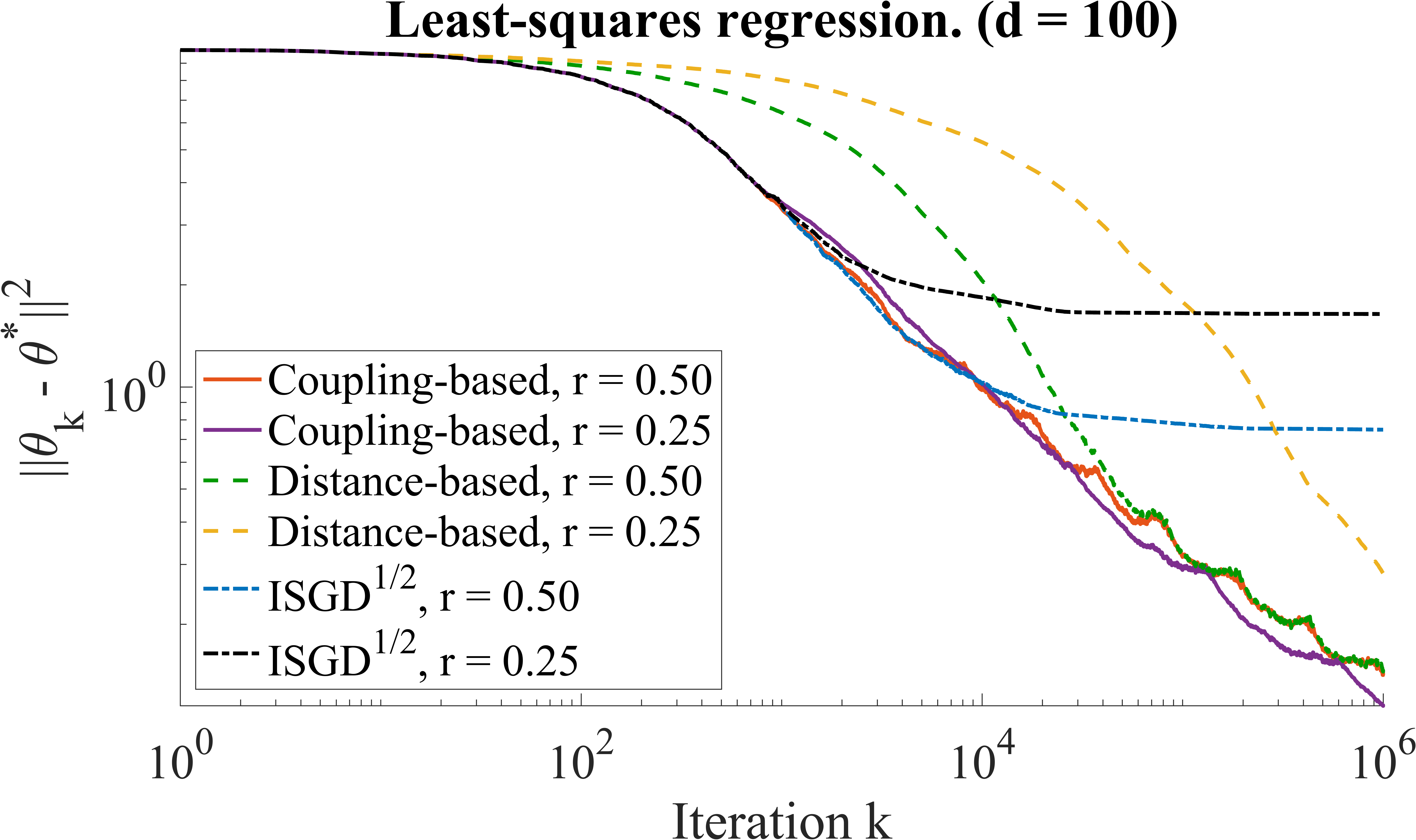}\label{fig: Robustness_LSR_r_adaptive}
    }
    \subfigure[Logistic regression: different initial threshold $\beta$]{
        \includegraphics[width=0.43\textwidth]{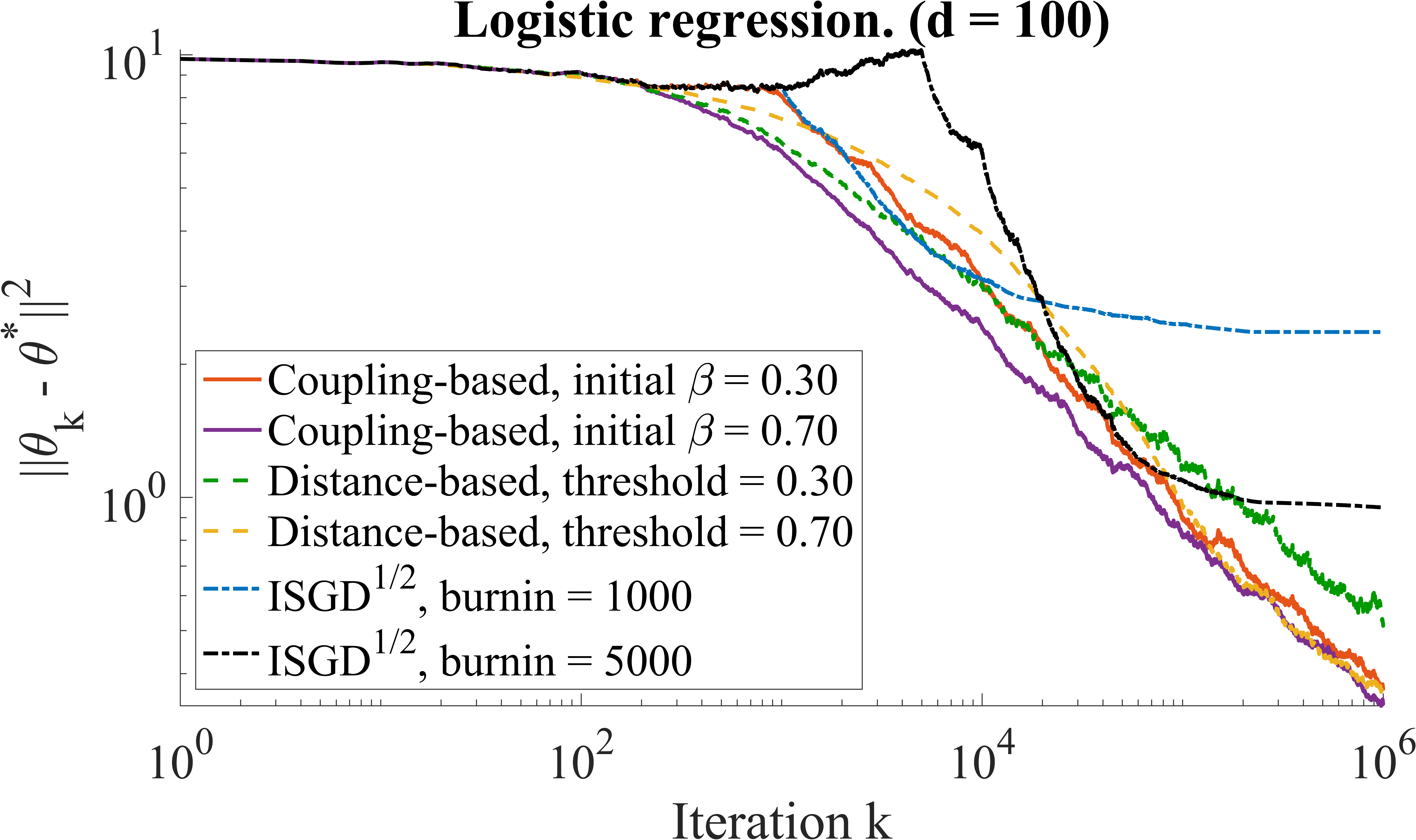} \label{fig: Robstuness_LR_thresh_adaptive}
    }
    \subfigure[LSR: different initial threshold $\beta$]{
        \includegraphics[width=0.43\textwidth]{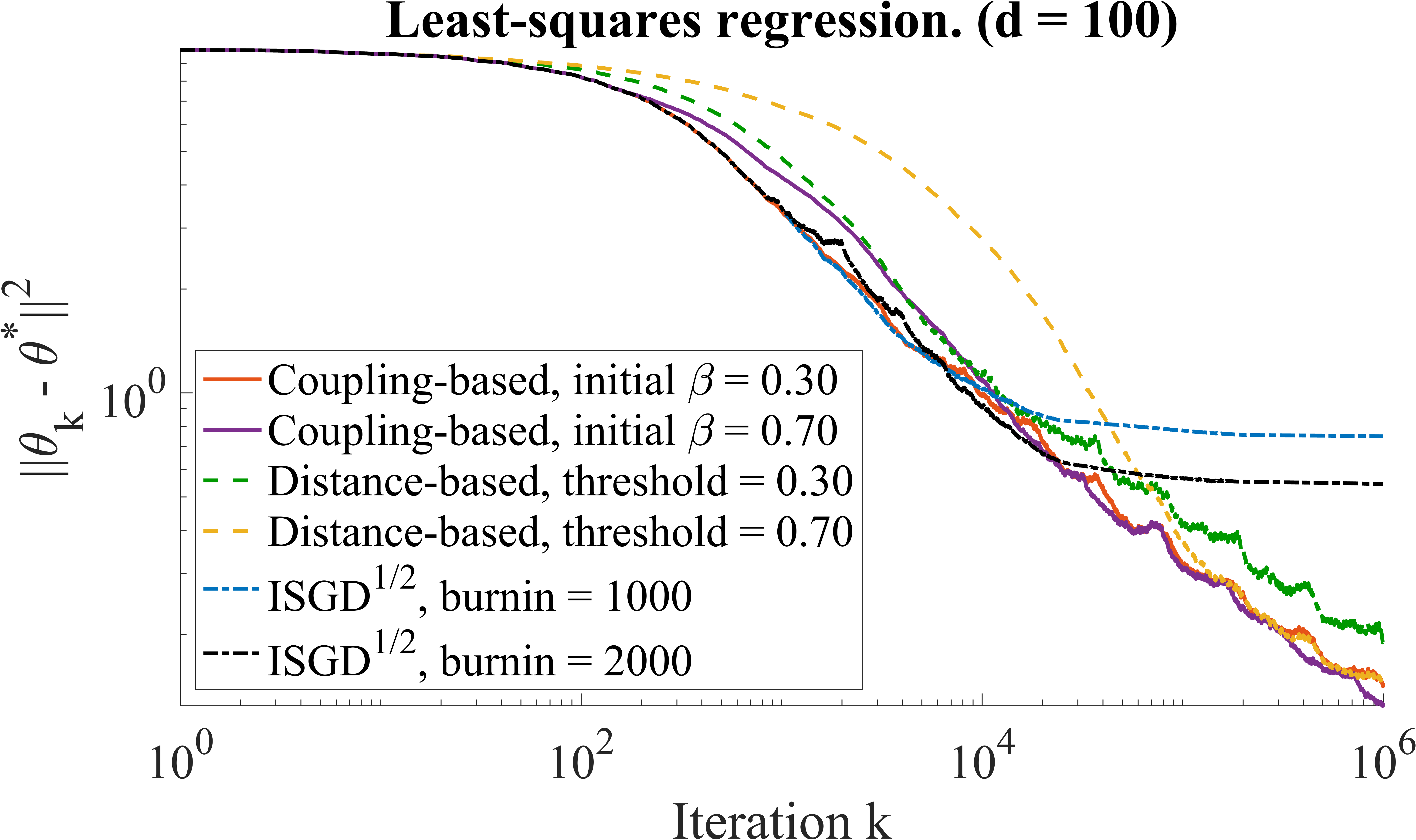} \label{fig: Robustness_LSR_thresh_adaptive}
    }
    \subfigure[Logistic regression: different back steps $b$]{
    \includegraphics[width=0.43\textwidth]{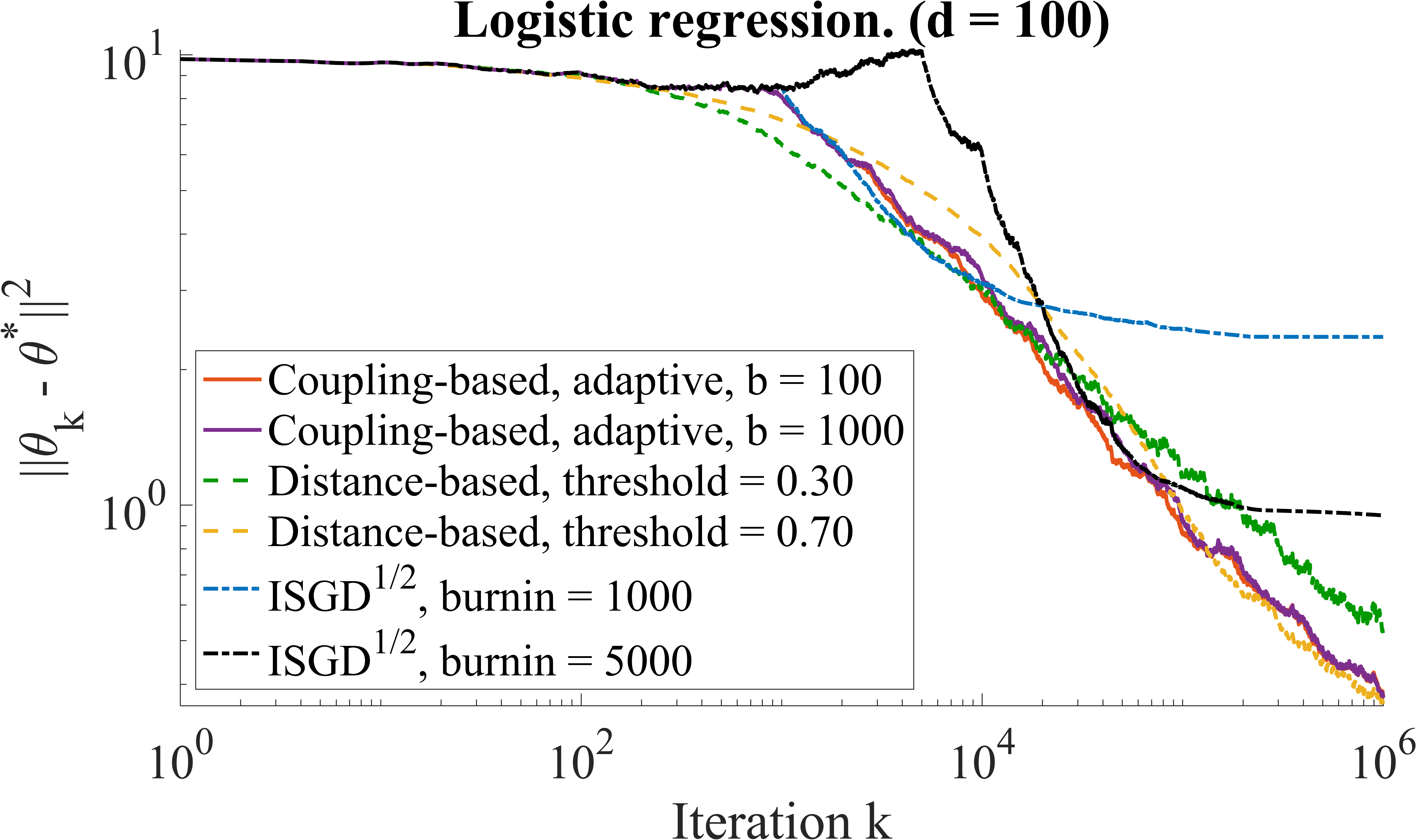}\label{fig: Robustness_LR_bstep_adaptive}
    }
    \subfigure[LSR: different back steps $b$]{
    \includegraphics[width=0.43\textwidth]{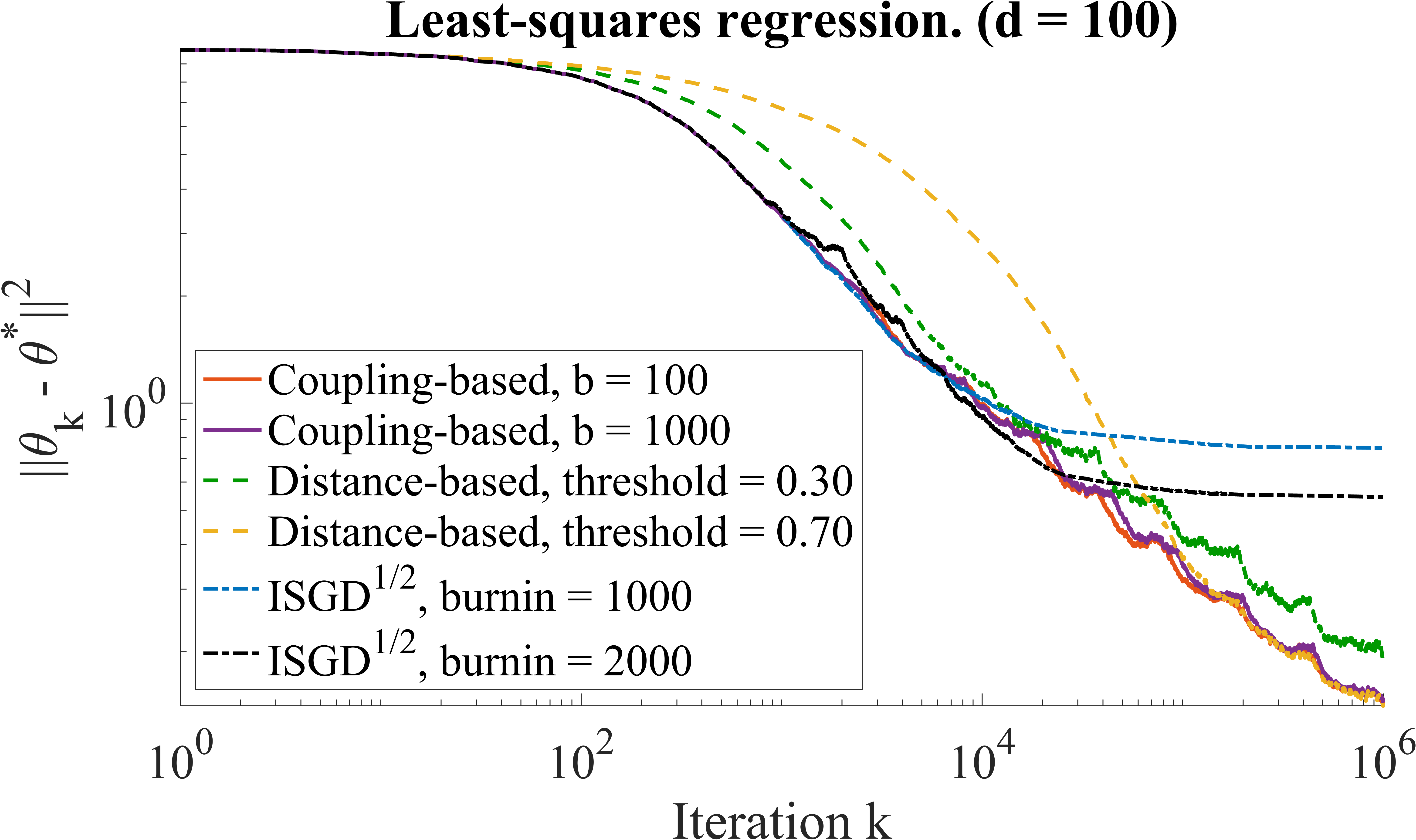}\label{fig: Robustness_LSR_bstep_adaptive}
    }
    \subfigure[Logistic regression: different  threshold decay factor $\eta$]{
    \includegraphics[width=0.43\textwidth]{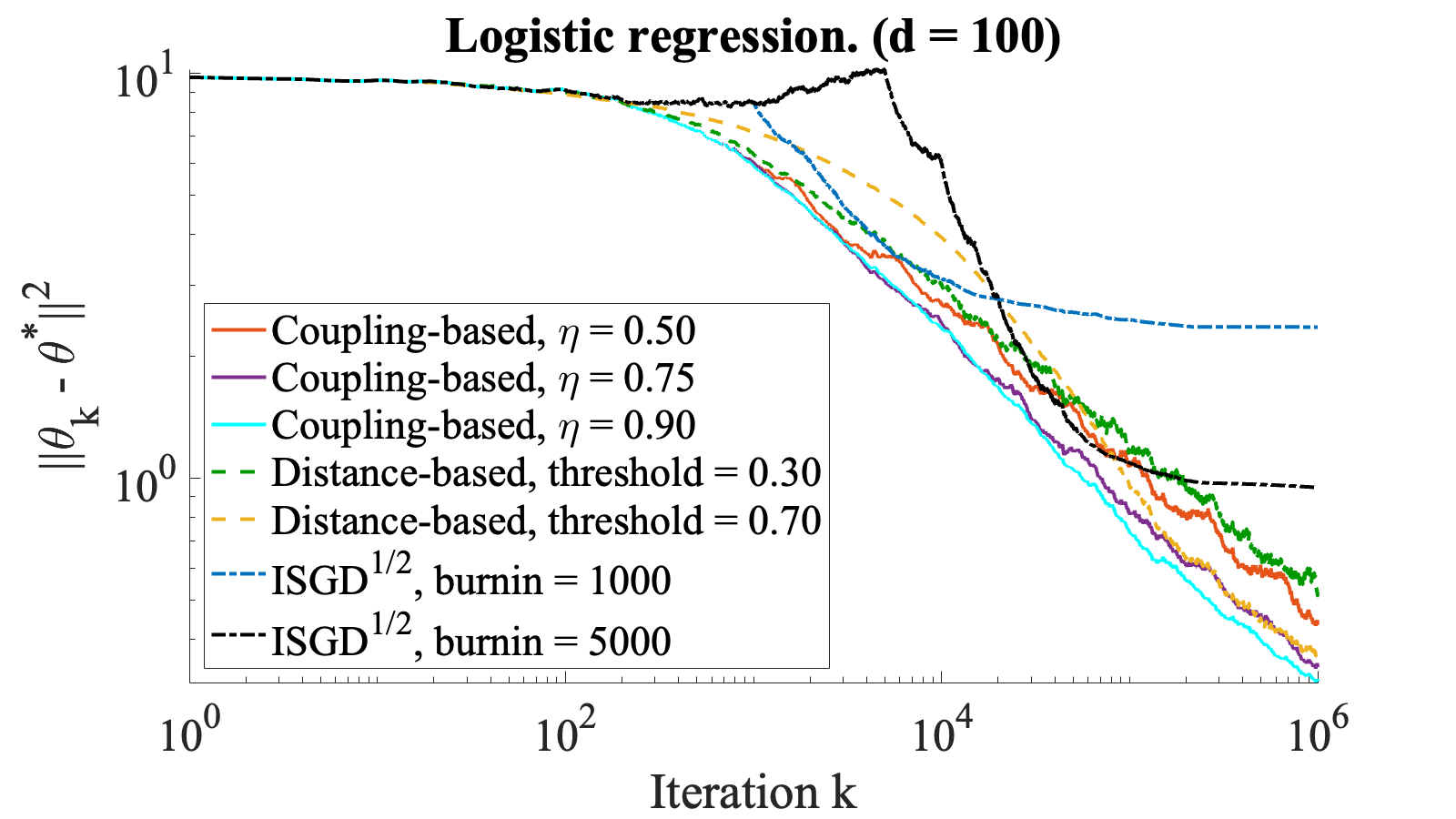}\label{fig: Robustness_LR_eta}
    }
    \subfigure[LSR: different  threshold decay factor $\eta$]{
    \includegraphics[width=0.43\textwidth]{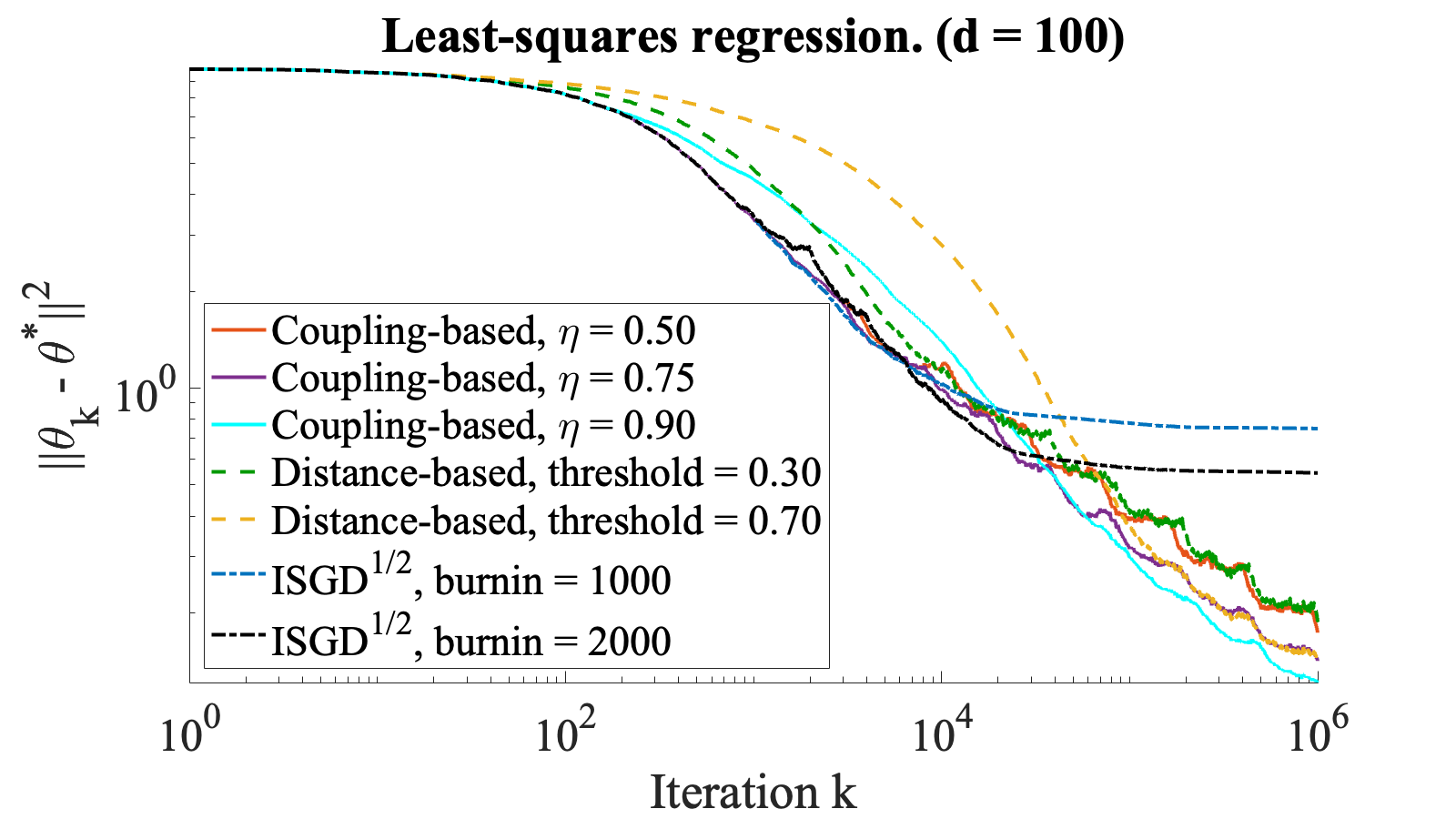}\label{fig: Robustness_LSR_eta}
    }
    \caption{Robustness results under least squares regression (LSR) and logistic regression with $d = 100$ for Algorithm \ref{alg:adaptive}.}
    \label{fig: Robustness_adaptive}
\end{figure*}

\subsection{Experiments on other settings}

To provide a comprehensive assessment on the effectiveness of our methods, we conduct more experiments on additional settings.

\paragraph{SVM.}
The objective function $f$ is given by $f(\theta) = \mathbb{E}\left[\max(0, 1 - y_i \langle x_i, \theta \rangle)\right] + \frac{\lambda}{2} \|\theta\|^2,$ where $ \lambda > 0$ is the regularization parameter. Note that $f$ is strongly convex with parameter $\lambda = 0.1$ and is non-smooth. The inputs $x_i$ are i.i.d.\ drawn from $\mathcal{N}(0, \sigma^2 \mathbf{I}_d)$, where $\mathbf{I}_d\in \R^{d\times d}$ is the identity matrix with $d=20$. The outputs $y_i$ are generated as $y_i = \text{sgn}(x_i^{(1)} + z_i)$, where $z_i \sim \mathcal{N}(0, \sigma^2)$. For additional baselines, we also include the averaged-SGD with stepsize $\gamma_k=1/{\lambda k},$ as well as the averaged-SGD with stepsize $\gamma_k=C/{\sqrt{k}}$ with the parameter $C$ tuned to achieve the best performance. 

\begin{figure}[ht]
    \centering
\includegraphics[width=0.95\columnwidth]{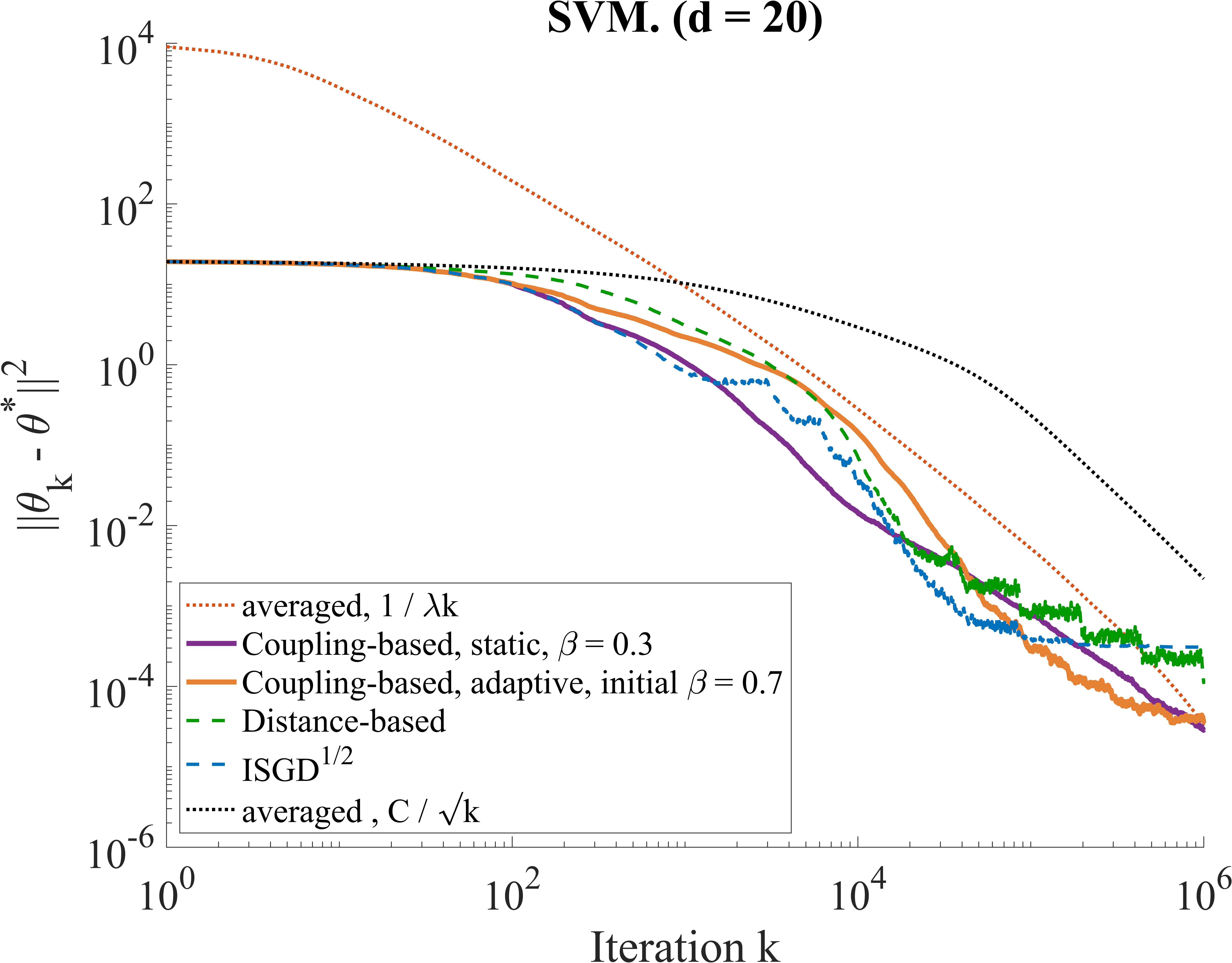}
    \caption{SVM with dimension $d = 20$. The initial stepsize of coupling/distance-based and $\text{ISGD}^{1/2}$ is $\gamma_0 = 4/R^2$. The errors are averaged over 10 replications.}
    \label{fig:svm}
\end{figure}

As shown in Figure \ref{fig:svm}, our coupling-based method achieves comparable performance to the averaged SGD with $\gamma_k = 1/\lambda k$ at the end. However, we observe that the latter one is compromised by a slow initial convergence rate. Additionally, we note that both the static and adaptive versions of our algorithm outperform the distance-based and the $\text{ISGD}^{1/2}$ method.

\paragraph{Uniformly convex $f$.} 
The objective function $f$ is given by $f(\theta) = \frac{1}{\rho}\|\theta\|^\rho_2$. We remark that $f$ is uniformly convex with parameter $\rho>2$ \cite{pesme2020distance}. We consider $\rho = 2.5$, and assume that the gradient noise $\xi_i$ are i.i.d.\ generated from $\mathcal{N}(0, \mathbf{I}_d)$ with $d=200$. We also consider SGD with stepsize $\gamma_k=1/{\sqrt{k}}$, which achieves the rate of $\mathcal{O}\left(\log(k)/\sqrt{k}\right)$ \cite{shamir2013stochastic}. In addition, we include SGD with stepsize $\gamma_k = k^{-1/(\tau+1)}$, which is conjectured to achieve the optimal rate of $\mathcal{O}(k^{-1/(\tau+1)}\log(k))$ \cite{pesme2020distance}, where $\tau = 1 - 2/\rho$.

\begin{figure}[ht]
    \centering
\includegraphics[width=0.95\columnwidth]{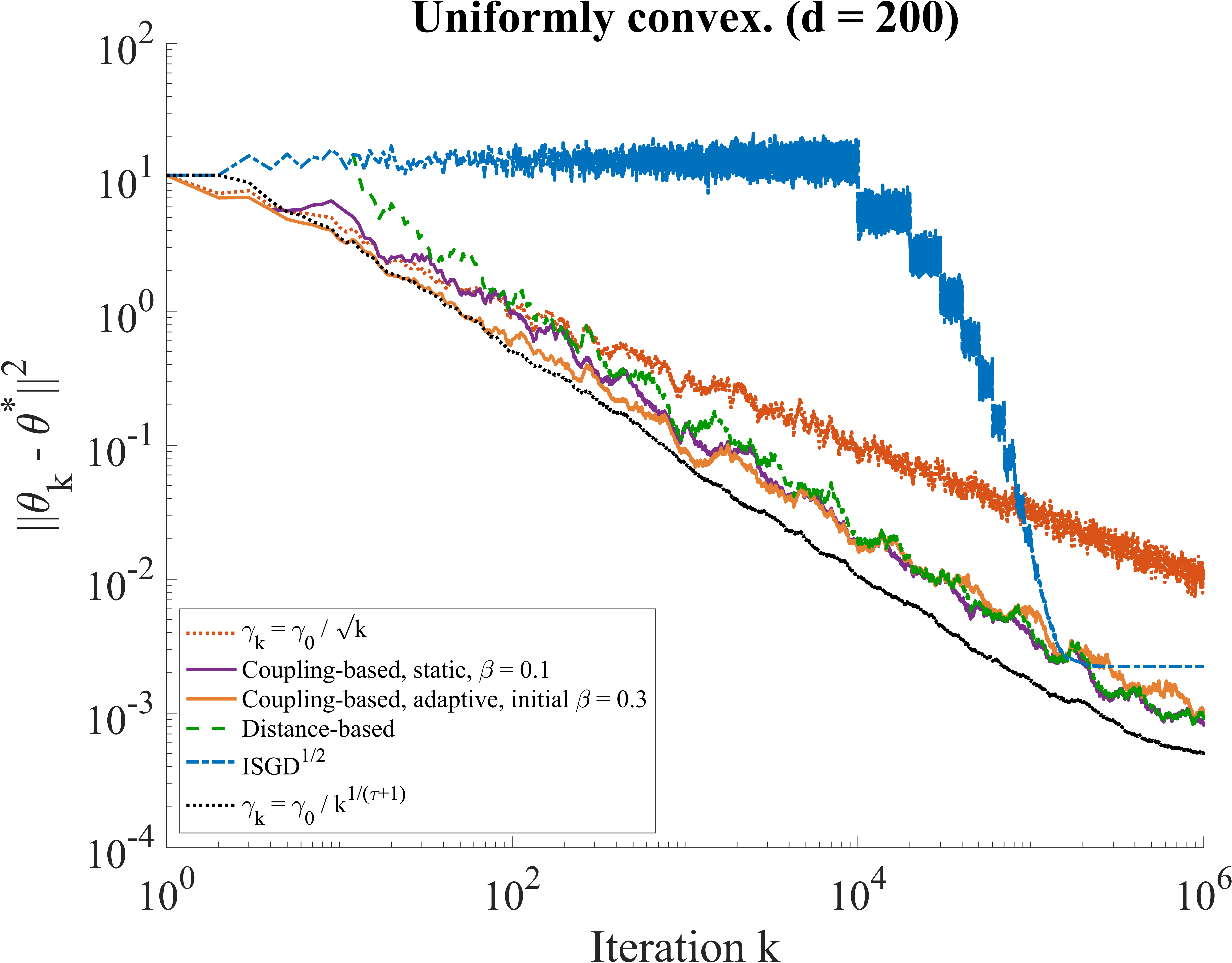}
    \caption{Uniformly convex objective with $d = 200$. The initial stepsize of coupling/distance-based and $\text{ISGD}^{1/2}$ is $\gamma_0 = 1/4L$. The errors are averaged over 10 replications.}
    \label{fig:uniformly}
\end{figure}

Figure \ref{fig:uniformly} shows that the Pflug's statistic-based $\text{ISGD}^{1/2}$ algorithm experiences delays in detecting convergence initially, resulting in the slowest convergence among all methods. We observe that the two variants of our coupling-based algorithm converge faster than other algorithms; moreover, our methods almost matches the optimal performance of SGD with a diminishing stepsize $\gamma_k = \gamma_0/k^{1/(\tau+1)}$.

\paragraph{Lasso regression.}
The objective function $f$ is given by $f(\theta) = \frac{1}{n}\sum_{i=1}^n(y_i-\langle x_i,\theta\rangle)^2 + \lambda\|\theta\|_1$, where $\lambda = 10^{-4}$. The inputs $x_i$ are i.i.d.\ generated in the same way as in the logistic model. The outputs $y_i$ are generated according to $y_i = \langle x_i\tilde\theta\rangle + \varepsilon_i$, where $\tilde{\theta}$ is an $s$-sparse vector with $s = 60$, and $\varepsilon_i$ are i.i.d.\ drawn from $\mathcal{N}(0, \sigma^2)$. We also evaluate SGD with stepsize $\gamma_k = 1/\sqrt{k}$, which achieves the optimal rate of $\mathcal{O}\left(\log(k)/\sqrt{k}\right)$ \cite{shamir2013stochastic}, and SGD with stepsize $\gamma_k=C/{\sqrt{k}}$ where $C$ is tuned to achieve the best performance. The initial stepsize for our algorithm, the distance-based method and the $\text{ISGD}^{1/2}$ are set as $1/2R^2.$

\begin{figure}[ht]
    \centering
\includegraphics[width=0.95\columnwidth]{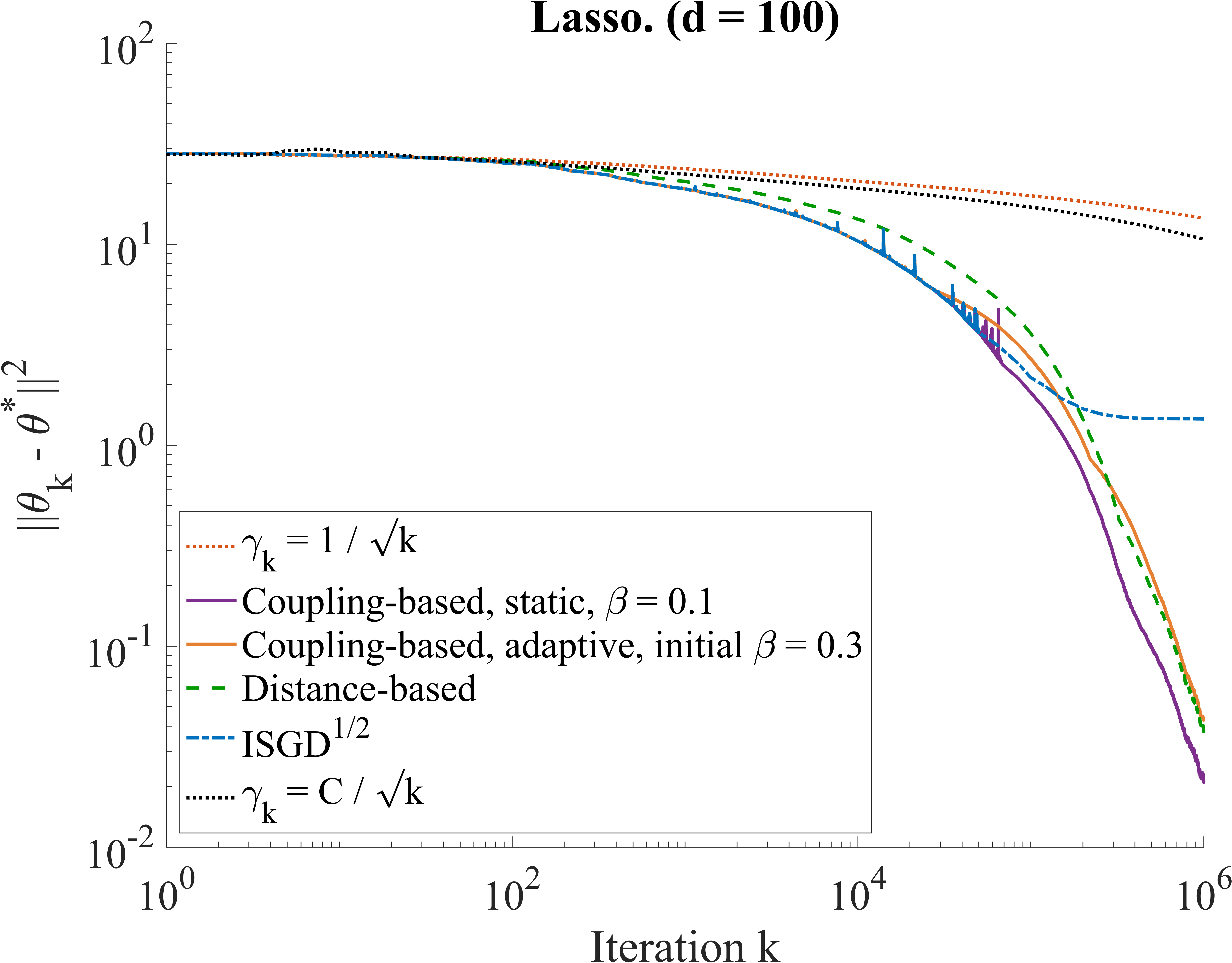}
    \caption{Lasso regression with dimension $d=100$. The initial stepsizes are $\gamma_0 = 1/2R^2$.}
    \label{fig:lasso}
\end{figure}

As shown in Figure \ref{fig:lasso}, both static and adaptive variants of our algorithm demonstrate superior performance. The $\text{ISGD}^{1/2}$ leads to a poor performance again with early saturation. 

\paragraph{Linear stochastic approximation with Markovian data.} Our last set of experiments consider the linear stochastic approximation (LSA) iteration driven by Markovian data: $\theta_{k+1}=\theta_k+\gamma\big( A(x_k)\theta_k+b(x_k)\big),$ where $(x_k)_{k\geq0}$ is a Markov chain denoting the underlying data stream, $A\in \R^{d\times d}$ and $b\in \R^d$ are deterministic functions. LSA aims to solve the linear point equation $\Bar{A}\theta^\star+\Bar{b}=0$ iteratively, where $\Bar{A}=\E_{x\sim \pi} [A(x)],\Bar{b}=\E_{x\sim \pi} [b(x)]$ with $\pi$ being the stationary distribution of the Markov chain $(x_k)_{k\geq0}.$  LSA covers the popular Temporal Difference (TD) learning algorithm in reinforcement learning and a variety of its variants, as well as SGD applied to a quadratic objective function. Recent work shows that LSA with a constant stepsize $\gamma>0$ also exhibits the transience-stationarity transition, with the saturation error in stationary phase induced by the Markovian data \cite{huo2023bias,lauand2023curse}. 

\begin{figure}[ht]
    \centering
\includegraphics[width=0.95\columnwidth]{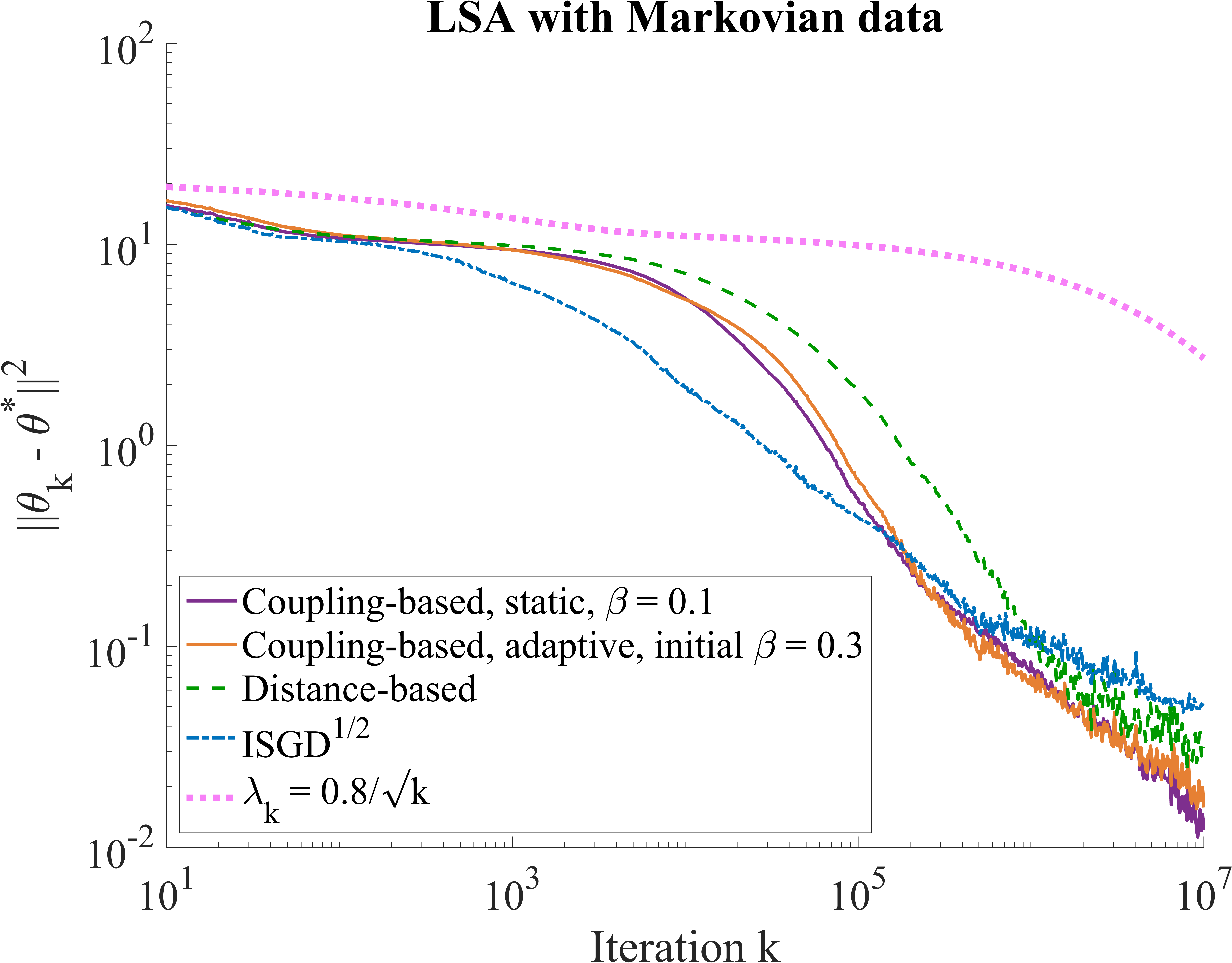}
    \caption{LSA with Markovian data stream.}
    \label{fig:LSA}
\end{figure}

We compare our algorithms with other approaches under the LSA setting with $d=5$, as shown in Figure \ref{fig:LSA}. We observe that the commonly employed diminishing stepsize $\gamma_k=1/\sqrt{k}$ converges slowly. Both the static and adaptive variants of our algorithms outperform the distance-based method. The $\text{ISGD}^{1/2}$ algorithm achieves fast convergence initially, but it is outperformed by our algorithms with increasing iterations. We remark that we tune the hyper-parameters of both the distance-based method and the $\text{ISGD}^{1/2}$ algorithm and present the best result here. Through our experiment, we find that $\text{ISGD}^{1/2}$ is very sensitive to its hyper-parameters, including the stepsize decay factor and the burnin parameter. In contrast, our algorithms achieve robust performance across a wide range of values for the hyper-parameters. 
Our results on LSA also showcase the applicability of our methods beyond the SGD setting.

\end{document}